\documentclass[review]{elsarticle}

\usepackage{graphicx} 
\usepackage{subfig}
\usepackage{amsfonts}
\usepackage{amsmath,amssymb}
\usepackage{amsthm}
\usepackage{enumerate}
\usepackage{lineno,hyperref}

\theoremstyle{remark}
\newtheorem{example}{Example}

\graphicspath{{figs/},{figs/network/},{figs/plane/},{figs/trajectories/data4/}}

\makeatletter
\def\ps@pprintTitle{%
  \let\@oddhead\@empty
  \let\@evenhead\@empty
  \let\@oddfoot\@empty
  \let\@evenfoot\@oddfoot
}
\makeatother

\newtheorem{definition}{Definition}
\newtheorem{theorem}{Theorem}
\newtheorem{lemma}{Lemma}

\newcommand{\real}{\ensuremath{\mathbb{R}}}
\newcommand{\ltwo}{\ensuremath{\mathbb{L}_2}}
\newcommand{\expect}[2]{\ensuremath{\mathbb{E}_{#1}\!\left[#2\right]}}

\newcommand{\borel}{\ensuremath{\mathcal{P}}}

\DeclareMathOperator{\Tr}{Tr}

\begin{document}

\begin{frontmatter}

\title{Probing the Geometry of Data with Diffusion Fr\'{e}chet Functions}

\author[ad1]{Diego Hern\'{a}n D\'{i}az Mart\'{i}nez}
\author[ad2]{Christine H. Lee}
\author[ad2,ad3]{Peter T. Kim}
\author[ad1]{Washington Mio}

\address[ad1]{Department of Mathematics, Florida State University, Tallahassee,
FL 32306-4510 USA}
\address[ad2]{Department of Pathology and Molecular Medicine,
McMaster University,
St Joseph's Healthcare,
50 Charlton Avenue E, 424 Luke Wing,
Hamilton ON L8N4A6 Canada}
\address[ad3]{Department of Mathematics and Statistics,
University of Guelph,
Guelph ON N1G 2W1 Canada}

%
%
%

\begin{abstract}
Many complex ecosystems, such as those formed by multiple microbial taxa, involve intricate interactions amongst various sub-communities. The most basic relationships are frequently modeled as co-occurrence networks in which the nodes represent the various players in the community and the weighted edges encode levels of interaction. In this setting, the composition of a community may be viewed as a probability distribution on the nodes of the network. This paper develops methods for modeling the organization of such data, as well as their Euclidean counterparts,
across spatial scales. Using the notion of diffusion distance, we introduce {\em diffusion Fr\'{e}chet functions\/} and  {\em diffusion Fr\'{e}chet vectors\/} associated with probability distributions on Euclidean space and the vertex set of a weighted network, respectively. We prove that these functional statistics are stable with respect to the Wasserstein distance between probability measures, thus yielding robust descriptors of their shapes. 

We apply the methodology to investigate bacterial communities in the human gut, seeking to characterize divergence from intestinal homeostasis  in patients with {\em Clostridium difficile\/} infection (CDI) and the effects of fecal microbiota transplantation, a treatment used in CDI patients that has proven to be significantly more effective than traditional treatment with antibiotics. The proposed method proves useful in deriving a biomarker that might help elucidate the mechanisms that drive these processes.
\end{abstract}


\end{frontmatter}

\section{Introduction}

The state of a complex ecosystem is frequently described by a probability
distribution on the nodes of a weighted network. For example, a microbial
community may be modeled on a network in which each node represents a
taxon and the main interactions between pairs of taxa are represented
by weighted edges, the weights reflecting the levels of interaction. In
this formulation, bacterial relative abundance in a sample may be viewed as a
probability distribution $\xi$ on the vertex set $V$ of the network.
Thus, $\xi$ describes the composition of the community whereas
the underlying network encodes the expected interactions
amongst the various taxa. Identifying sub-communities and their 
organization at different scales from such data, quantifying variation
across samples, and integrating information across scales are core
problems. Similar questions arise in other contexts such as
probability measures defined on Euclidean spaces or other
metric spaces. To address these problems, we develop methods that
(i) capture the geometry of data and probability measures across
a continuum of spatial scales and (ii) integrate information obtained
at all scales. In this paper, we focus on the Euclidean and network cases. 
We analyze distributions with the aid of a 1-parameter family of diffusion
metrics $d_t$, $t>0$, where $t$ is treated as the scale parameter
\cite{coifman}. For (Borel) probability measures $\alpha$ on $\real^d$,
our approach is based on a functional statistic
$V_{\alpha, t} \colon \real^d \to \real$ derived from $d_t$ and termed
the {\em diffusion Fr\'{e}chet function\/} of $\alpha$ at scale $t$. Similarly,
we define {\em diffusion Fr\'{e}chet vectors} $F_{\xi, t}$ associated
with a  distribution $\xi$ on the vertex set of a network. 

The (classical) Fr\'{e}chet function of a random variable $y \in \real^d$ 
distributed according to a probability measure $\alpha$ with
finite second moment is defined as
\begin{equation} \label{E:frechet1}
V_\alpha (x) = \expect{\alpha}{\|y-x\|^2}
= \int_{\real^d} \|y-x\|^2 \, d\alpha (y) \,.
\end{equation} 
The Fr\'{e}chet function quantifies the scatter of $y$ about the point $x$
and depends only on $\alpha$. It is well-known that the {\em mean\/}
(or {\em expected value}) of $y$ is the unique minimizer of $V_\alpha$.
For complex distributions, aiming at more effective descriptors,
we introduce diffusion Fr\'{e}chet functions and show that they encode
a wealth of information about the shape of $\alpha$. At scale $t>0$,
we define the {\em diffusion Fr\'{e}chet function\/}
$V_{\alpha,t} \colon \real^d \to \real$ by
\begin{equation} \label{E:frechet2}
V_{\alpha,t} (x) = \expect{\alpha}{d_t^2 (y,x)} = \int_{\real^d}
d_t^2 (y,x) \, d\alpha (y) \,.
\end{equation} 
This definition simply replaces the Euclidean distance in
\eqref{E:frechet1} with the diffusion distance $d_t$. $V_{\alpha,t} (x)$
may be viewed as a localized second moment of $\alpha$ about $x$.
In the network setting, Fr\'{e}chet vectors are defined in a similar manner
through the diffusion kernel associated with the graph Laplacian.
Unlike $V_\alpha$ that is a quadratic function, diffusion Fr\'{e}chet functions
and vectors typically exhibit rich profiles that reflect the shape of the distribution at
different scales. We illustrate this point with simulated data
and exploit it in an application to the analysis
of microbiome data associated with {\em Clostridium difficile} infection (CDI).
On the theoretical front, we prove stability theorems for diffusion Fr\'{e}chet
functions and vectors with respect to the Wasserstein distance between
probability measures \cite{villani09}. The stability results ensure that both
$V_{\alpha, t}$ and $F_{\xi, t}$ yield
robust descriptors, useful for data analysis. 

As explained in detail below, $V_{\alpha,t}$ is closely related to the
solution of the heat equation $\partial_t u = \Delta u$ with initial condition
$\alpha$. In particular, if $p_1, \ldots, p_n \in \real^d$ are data
points sampled from $\alpha$, then the diffusion Fr\'{e}chet functions for
the empirical measure  $\alpha_n = \frac{1}{n} \sum_{i=1}^n \delta_{p_i}$
give a reinterpretation of Gaussian density estimators derived from
the data as localized second moments of $\alpha_n$. 
An interesting consequence of this fact is that, for the Gaussian
kernel, the scale-space model for data on the real line investigated
in \cite{scalespace} may be recast as a 1-parameter family of
diffusion Fr\'{e}chet functions.

The rest of the paper is organized as follows. Section \ref{S:distance}
reviews the concept of diffusion distance in the Euclidean and network
settings. We define multiscale diffusion Fr\'{e}chet functions in Section
\ref{S:efrechet} and prove that they are stable with respect to the
Wasserstein distance between probability measures defined in Euclidean spaces. The network
counterpart is developed in Section \ref{S:nfrechet}. In 
Section \ref{S:cdi} we describe and analyze microbiome data
associated with {\em C. difficile} infection. We
conclude with a summary and some discussion in
Section \ref{S:remarks}.

\section{Diffusion Distances} \label{S:distance} 

Our multiscale approach to probability measures uses a formulation
derived from the notion of diffusion distance \cite{coifman}. In
this section, we review diffusion distances associated with the heat
kernel on $\real^d$, followed by a discrete analogue for weighted
networks.

For $t>0$, let $G_t \colon \real^d \times \real^d \to \real$
be the diffusion (heat) kernel
\begin{equation} \label{E:gauss}
G_t (x,y) = \frac{1}{C_d (t)} \exp \left( -\frac{\|y-x\|^2}{4t} \right) \,,
\end{equation}
where $C_d (t) = (4 \pi t)^{d/2}$. If an initial distribution of
mass on $\real^d$ is described by a Borel probability measure $\alpha$
and its time evolution is governed by the heat equation
$\partial_t u = \Delta u$, then the distribution at time $t$ has a smooth
density function $\alpha_t (y) = u (y,t)$ given by the convolution
\begin{equation}
u (y, t) = \int_{\real^d} G_t (x,y) \, d\alpha (x) \,.
\end{equation} For an initial point mass located at $x \in \real$,
$\alpha$ is Dirac's delta $\delta_x$ and $u (y,t) = G_t(x, y)$.

The diffusion map $\Psi_t \colon \real^d \to \ltwo (\real^d)$,
given by $\Psi_t (x) = G_t (x, \cdot)$, embeds $\real^d$ into
$\ltwo (\real^d)$. The {\em diffusion distance} $d_t$ on $\real^d$ is
defined as
\begin{equation}
d_t (x_1, x_2) = \|\Psi_t (x_1) - \Psi_2 (x_2) \|_2 \,,
\end{equation} the metric that $\real^d$ inherits from $\ltwo (\real^d)$
via the embedding $\psi_t$.  A calculation shows that 
\begin{equation} \label{E:difdistance}
\begin{split}
d^2_t (x_1, x_2) &=  \|\Psi_t (x_1)\|^2_2  + \|\Psi_t (x_2)\|^2_2 
- 2G_{2t}(x_1,x_2) \\
&= 2 \left( \frac{1}{C_d (2t)} - G_{2t} (x_1, x_2) \right) \,,
\end{split}
\end{equation} 
where we used the fact that $\|\Psi_t (x)\|^2_2 =
1/ C_d (2t)$, for every $x \in \real^d$. Note that this implies that
the metric space $(\real^d, d_t)$ has finite diameter. 

Diffusion distances on weighted networks may be defined in a
similar way by invoking the graph Laplacian and the associated
diffusion kernel.  Let $v_1, \ldots, v_n$ be the nodes of a weighted
network $K$. The weight of the edge between $v_i$ and $v_j$ is
denoted $w_{ij}$, with the convention that $w_{ij} = 0$ if there is
no edge between the nodes. We let $W$ be the $n\times n$
matrix whose $(i,j)$-entry is $w_{ij}$.  The graph Laplacian is the
$n \times n$ matrix $\Delta = D-W$, where $D$ is the diagonal matrix with
$d_{ii}=\sum_{k=1}^{n}w_{ik}$, the sum of the weights of all
edges incident with the node $v_i$ (cf.\,\cite{lux}). This definition is based
on a finite difference discretization of the Laplacian, except for a sign
that makes $\Delta$ a positive semi-definite, symmetric matrix. 
The diffusion (or heat) kernel  at $t>0$ is the matrix
$e^{-t\Delta}$.

Let $\xi$ be a probability distribution on the vertex set $V$ of $K$.
If $\xi_i \geq 0$, $1 \leq i \leq n$, is the probability of the vertex $v_i$,
we write $\xi$ as the vector $\xi = [\xi_1 \, \ldots \, \xi_n]^T \in \real^n$,
where $T$ denotes transposition. Clearly, $\xi_1 + \ldots + \xi_n =1$.
Note that $u = e^{-t\Delta} \xi$ solves the heat equation
$\partial_t u = - \Delta u$ with initial condition $\xi$. (The negative sign
is due to the convention made in the definition of $\Delta$.) Mass initially
distributed according to $\xi$, whose diffusion is governed by the heat equation,
has time $t$ distribution $u_t = e^{-t\Delta} \xi$. A point mass at the
$i$th vertex is described by the vector $e_i \in \real^n$, whose $j$th entry
is $\delta_{ij}$. Thus, $e^{-t\Delta} e_i$ may be viewed as a network
analogue of the Gaussian $G_t (x, \cdot)$. For $t>0$, the diffusion
mapping $\psi_t \colon V \to \real^n$ is defined by
$\psi_t (v_i) = e^{-t\Delta} e_i$ and the time $t$ {\em diffusion distance\/}
between the vertices $v_i$ and $v_j$ by
\begin{equation}
d_t (i,j) = \|e^{-t\Delta} e_i - e^{-t\Delta} e_j\| \,,
\end{equation}
where $\|\cdot\|$ denotes Euclidean norm. If 
$0 = \lambda_1 \leq \ldots \leq \lambda_n$ are the eigenvalues of $\Delta$
with orthonormal eigenvectors $\phi_1, \ldots, \phi_n$, then
\begin{equation}
d_{t}^2(i,j)= \sum_{k=1}^{n}e^{-2\lambda_kt}\left(\phi_k(i)-\phi_k(j)\right)^2 \,,
\end{equation} where $\phi_k (i)$ denotes the $i$th component of $\phi_k$ \cite{coifman}. 


\section{Diffusion Fr\'{e}chet Functions} 
\label{S:efrechet}


The classical Fr\'{e}chet function $V_\alpha$ of a probability
measure $\alpha$ on $\real^d$
with finite second moment is a useful statistic. However, for complex
distributions, such as those with a multimodal profile or more intricate
geometry, their Fr\'{e}chet functions usually fail to provide a good
description of their shape. Aiming at more effective descriptors,
we introduce diffusion Fr\'{e}chet functions and show that they encode
a wealth of information across a full range of spatial scales. At scale
$t>0$, define the {\em diffusion Fr\'{e}chet function\/}
$V_{\alpha,t} \colon \real^d \to \real$ by
\begin{equation} 
V_{\alpha,t} (x) = \int_{\real^d}
d_t^2 (y,x) \, d\alpha (y) \,.
\end{equation} 
Note that $V_{\alpha,t}$ is defined
for any Borel probability measure $\alpha$, not just those with finite
second moment, because $(\real^d, d_t)$ has finite diameter.  Moreover,
$V_{\alpha,t}$ is uniformly bounded. We also point out that this
construction is distinct from the standard kernel trick
that pushes $\alpha$ forward to a probability measure $(\psi_t)_\ast (\alpha)$
on $\ltwo (\real^d)$, where the usual Fr\'{e}chet function of
$(\psi_t)_\ast (\alpha)$ can be used to define such notions as an
extrinsic mean of $\alpha$. $V_t$ is an intrinsic statistic and typically a
function on $\real^d$ with a complex profile, as we shall see in the
examples below.
From \eqref{E:frechet2} and \eqref{E:difdistance}, it follows that
\begin{equation} \label{E:frechet3}
\begin{split}
V_{\alpha,t/2} (x) = \frac{2}{C_d (t)} - 2 \int_{\real^d}
G_{t} (x,y) \, d \alpha (y)
= \frac{2}{C_d (t)} - 2  \alpha_{t} (x) \,,
\end{split}
\end{equation}
where $\alpha_{t} (x) = u(x, t)$, the solution of the heat equation
$\partial_t u = \Delta u$ with initial condition $\alpha$.

If $p_1, \ldots, p_n \in \real^d$ are data points sampled from
$\alpha$, then the diffusion Fr\'{e}chet function of
the empirical measure $\alpha_n = \frac{1}{n} \sum_{i=1}^n \delta_{p_i}$
is closely related to the Gaussian density estimator
$\widehat{\alpha}_{n, t} = \frac{1}{n} \sum_{i=1}^n G_{t} (x, p_i)$
derived from the sample points.
Indeed, it follows from \eqref{E:frechet3} that
\begin{equation}
\begin{split}
V_{\alpha_n,t/2} (x) = \frac{2}{C_d (t)} - 2 \int_{\real^d}
G_{t} (x,y) \, d \alpha_n (y)
= \frac{2}{C_d (t)} - 2  \widehat{\alpha}_{n,t} (x) \,. 
\end{split}
\end{equation}
Thus, diffusion Fr\'{e}chet functions provide a new interpretation
of Gaussian density estimators, essentially as second moments
with respect to diffusion distances. This opens up interesting new
perspectives. For example, the
classical Fr\'{e}chet function $V_\alpha \colon \real^d \to \real$
(see \eqref{E:frechet1}) may be viewed as the trace of the
covariance tensor field $\Sigma^\alpha \colon \real^d \to \real^d
\otimes \real^d$ given by
\begin{equation}
\Sigma_\alpha (x) = \expect{\alpha}{(y-x) \otimes (y-x)}
= \int_{\real^d} (y-x)\otimes (y-x) \, d \alpha (y) \,.
\end{equation}
Thus, it is natural to ask: {\em How to define diffusion covariance
tensor fields $\Sigma_{\alpha,t}$ that capture the
modes of variation of $\alpha$, about each $x \in \real^d$ at all
scales, bearing a close relationship to $V_{\alpha,t}$?}
A multiscale approach to data along related lines has been
developed in \cite{mmm13,diaz1}.

\begin{example}
Here we consider the dataset highlighted in blue in
Figure \ref{F:evolution}, comprising $n=400$ points
$p_1, \ldots, p_n$ on the real line, grouped into two clusters. The
figure shows the evolution of the diffusion Fr\'{e}chet function
for the empirical measure $\alpha_n = \sum_{i=1}^n \delta_{p_i} /n$
for increasing values of the scale parameter. 
\begin{figure}[h!]
\begin{center}
\begin{tabular}{cc}
   \includegraphics[width=0.4\linewidth]{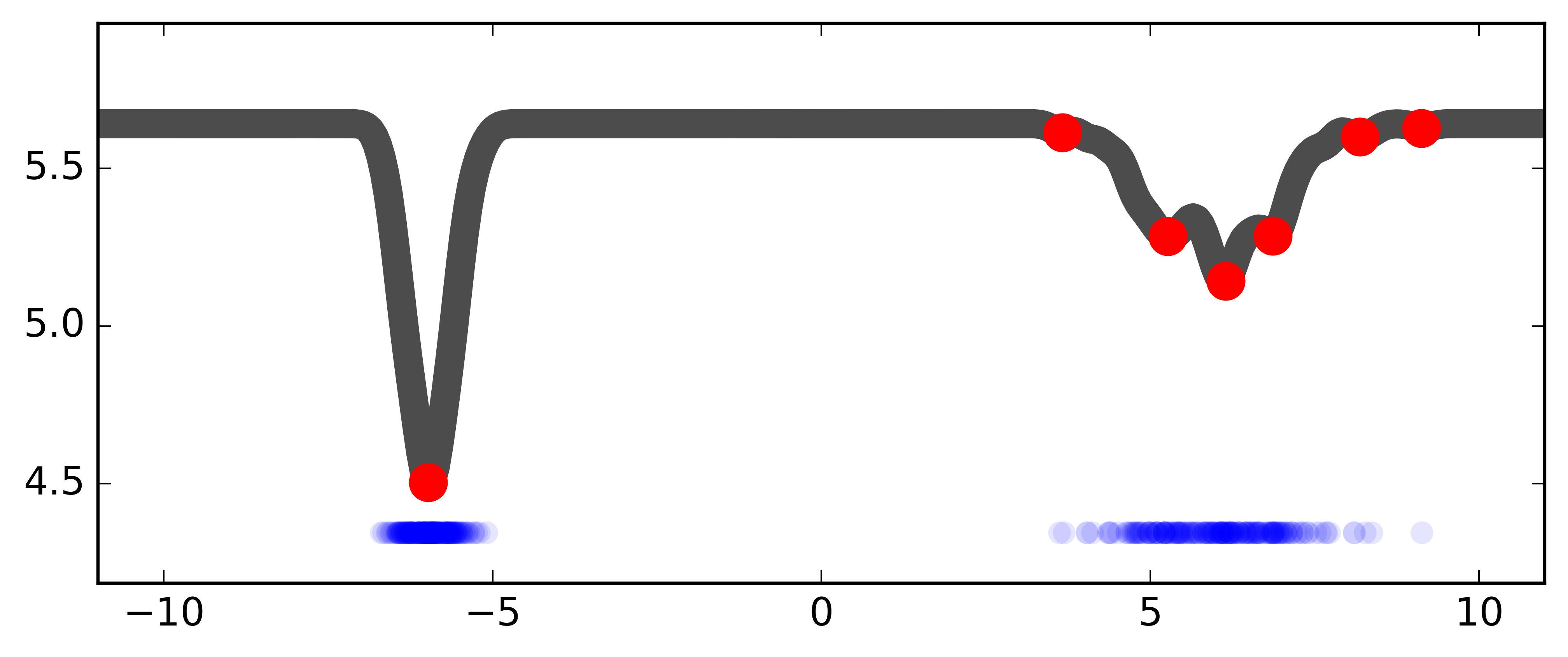} 
   \quad & \quad
   \includegraphics[width=0.4\linewidth]{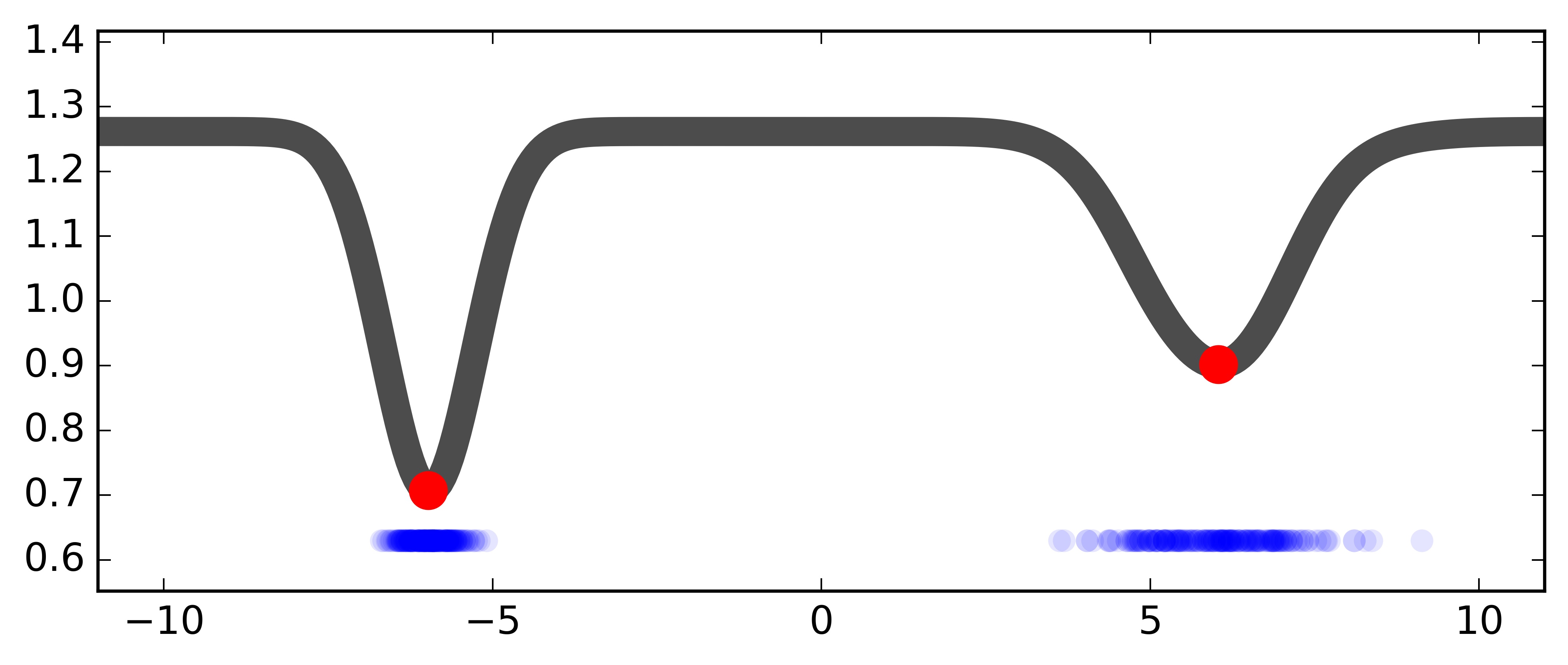} \\
   $t = 0.005$ \quad & \quad $t = 0.1$ \vspace{0.1in} \\
   \includegraphics[width=0.4\linewidth]{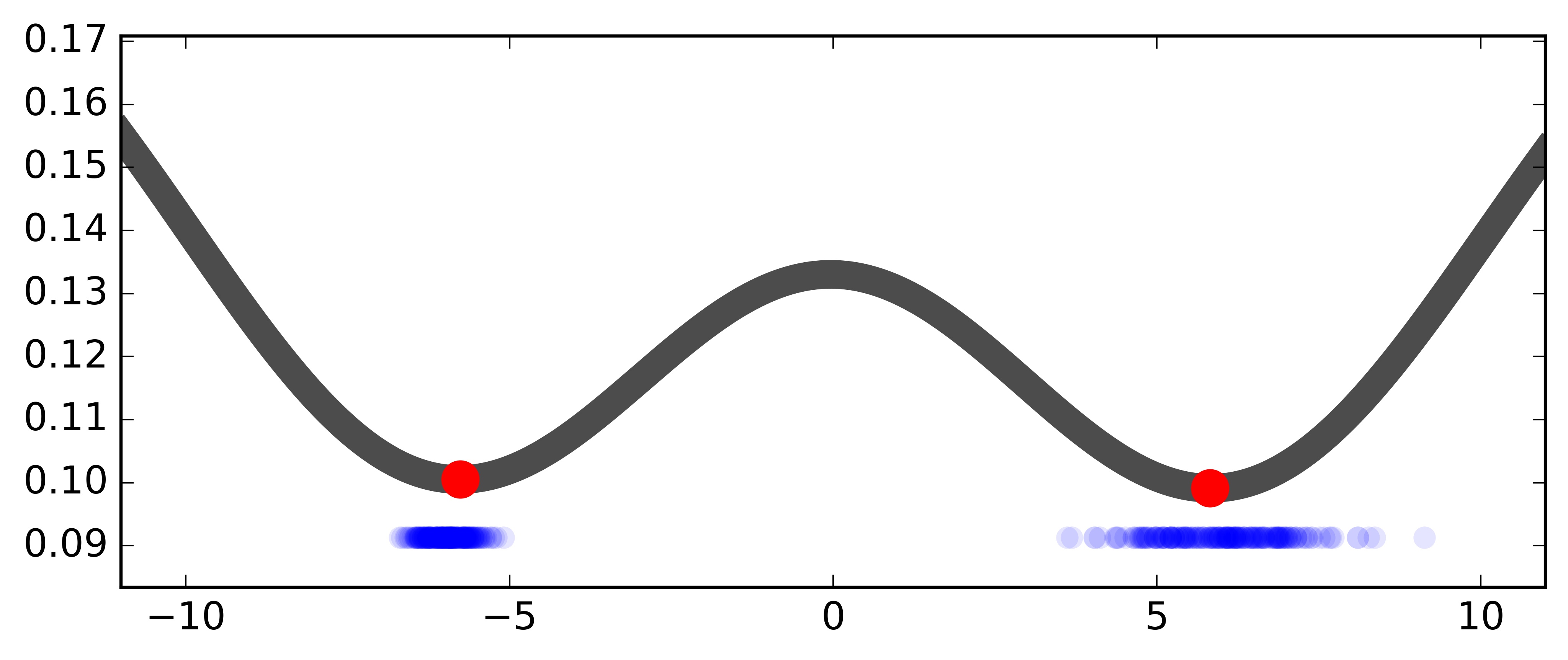} 
   \quad & \quad
   \includegraphics[width=0.4\linewidth]{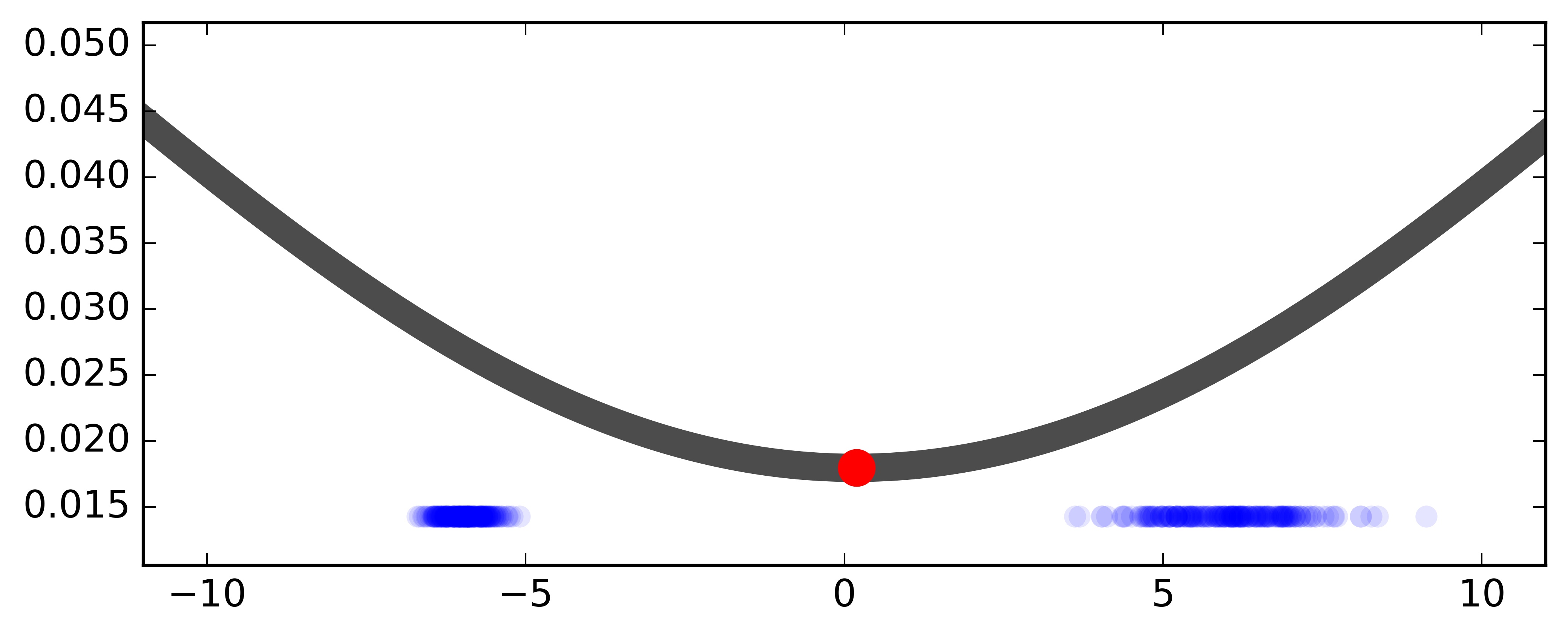} \\
   $t = 4$ \quad & \quad $t= 20$ \\
\end{tabular}
\end{center}
\caption{Evolution of the diffusion Fr\'{e}chet function across
scales.}
\label{F:evolution}
\end{figure}
At scale $t=0.1$, the Fr\'{e}chet function has two well defined
``valleys'', each corresponding to a data cluster. At smaller scales,
$V_{\alpha_n,t}$ captures finer properties of the organization
of the data, whereas
at larger scales $V_{\alpha_n,t}$ essentially views the data as
comprising a single cluster. For probability distributions on the real line,
the Fr\'{e}chet function levels off to $1/\sqrt{2 \pi t}$, as $|x| \to \infty$.
The local minima of $V_{\alpha,t}$ provide a generalization of the
mean of the distribution and a summary of the organization of the data
into sub-communities across multiple scales. Elucidating such
organization in data in an easily interpretable manner is one of our
principal aims.  
\end{example}

\begin{example}
In this example, we consider the dataset formed by $n = 1000$
points in $\real^2$, shown on the first panel of Figure\,\ref{F:frechet2d}.
The other panels show the diffusion Fr\'{e}chet functions for the
corresponding empirical measure $\alpha_n$, calculated at increasing
scales, and the gradient field $-\nabla V_{\alpha_n}$ at $t=2$,
whose behavior reflects the  organization of the data at that scale. 
\begin{figure}[h!]
\begin{center}
\begin{tabular}{cc}
\begin{tabular}{ccc}
\begin{tabular}{c}
\includegraphics[width=0.14\linewidth]{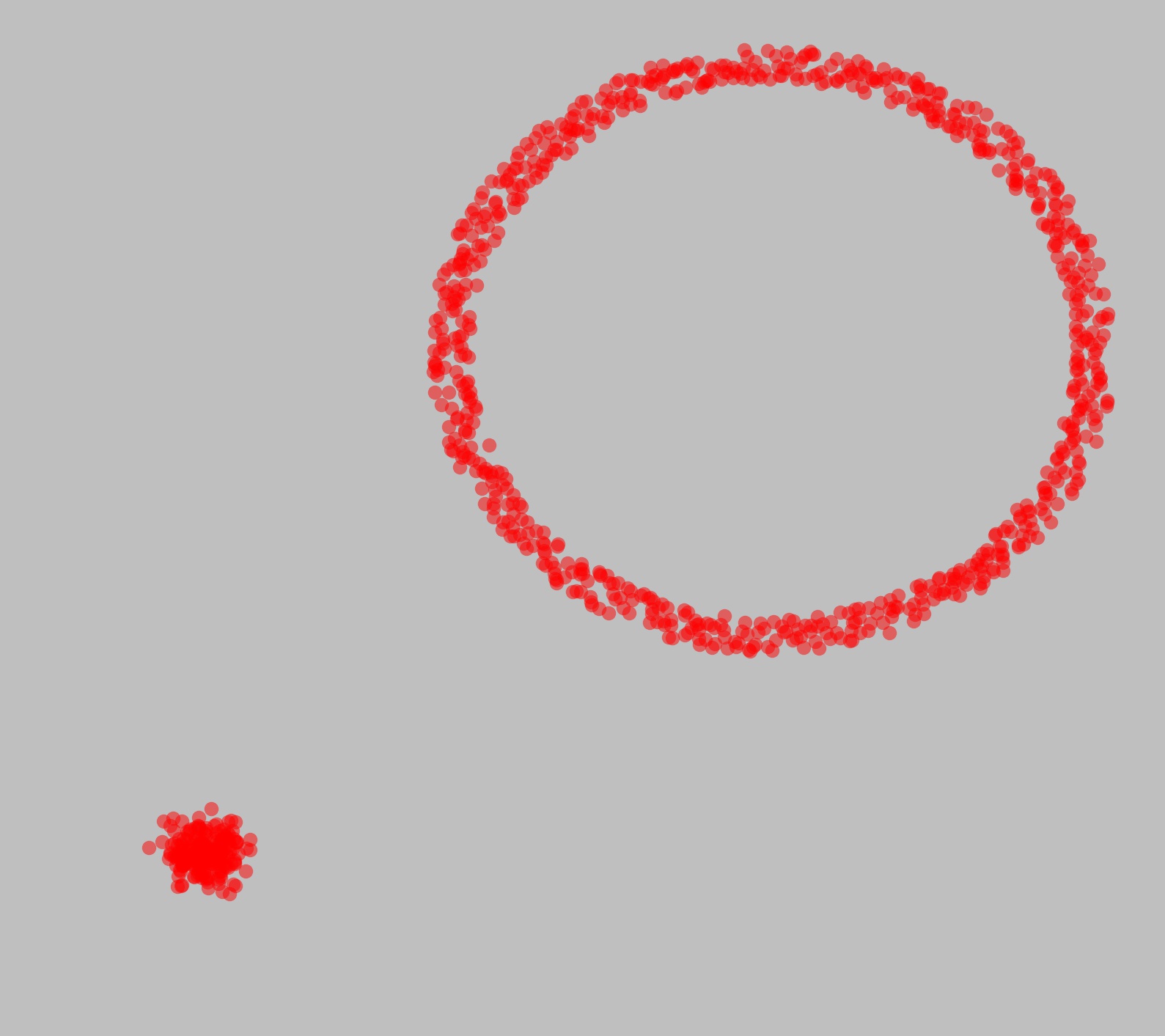}
\end{tabular} & 
\begin{tabular}{c}
\includegraphics[width=0.14\linewidth]{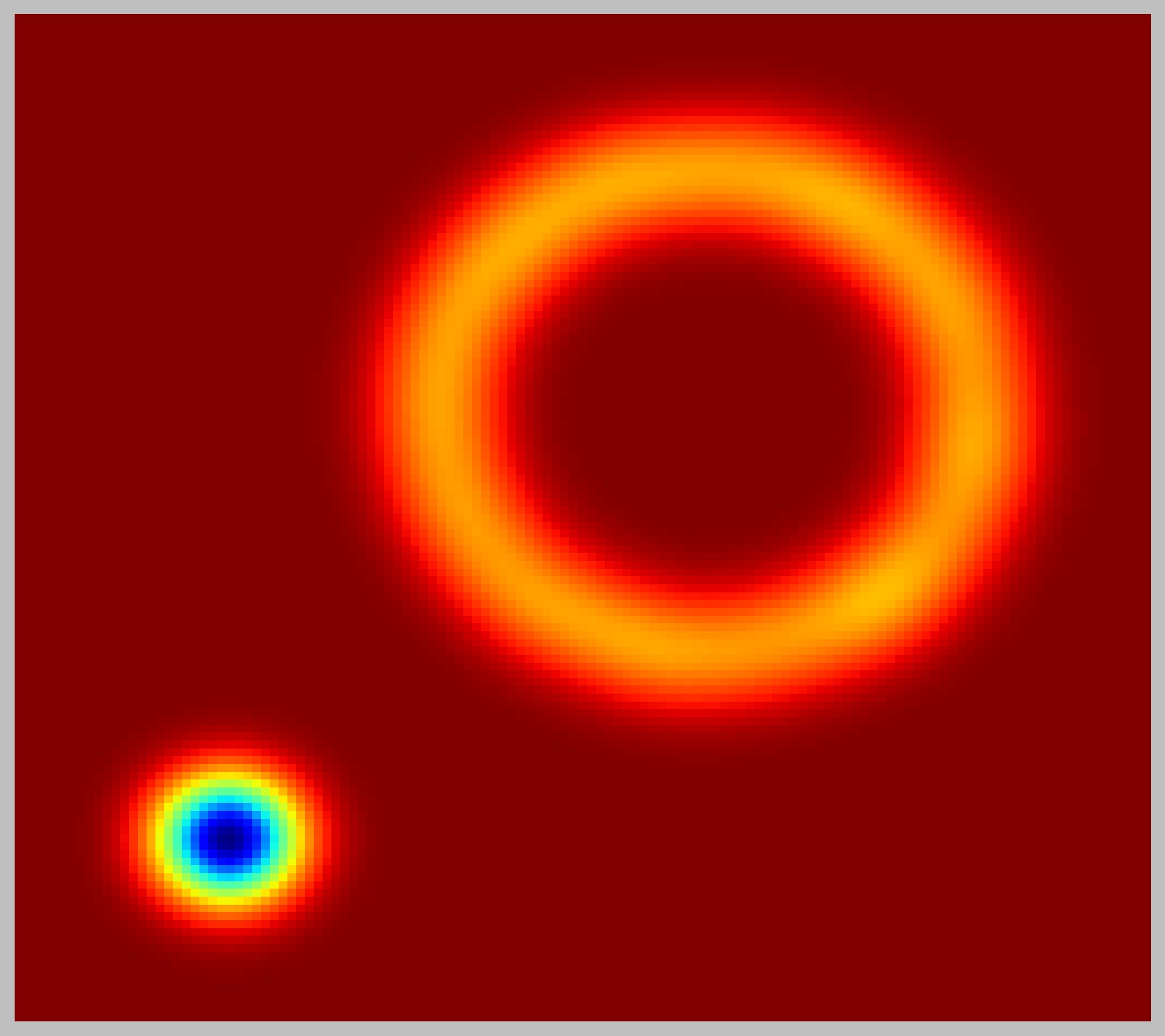} 
\end{tabular} & 
\begin{tabular}{c}
\includegraphics[width=0.14\linewidth]{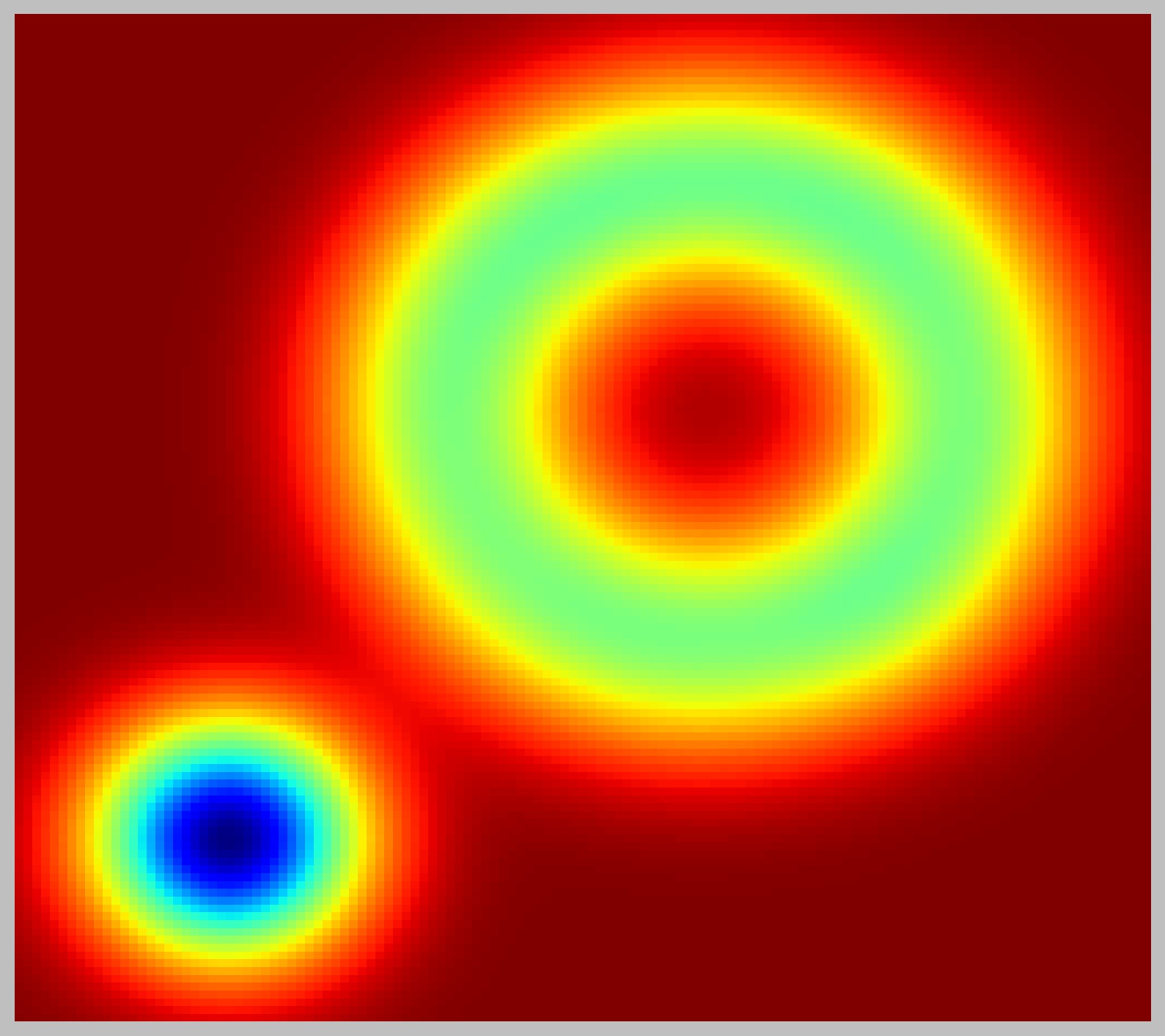} 
\end{tabular} \\
data & $t = 0.5$ & $t = 2$
\vspace{0.1in} \\
\begin{tabular}{c}
\includegraphics[width=0.14\linewidth]{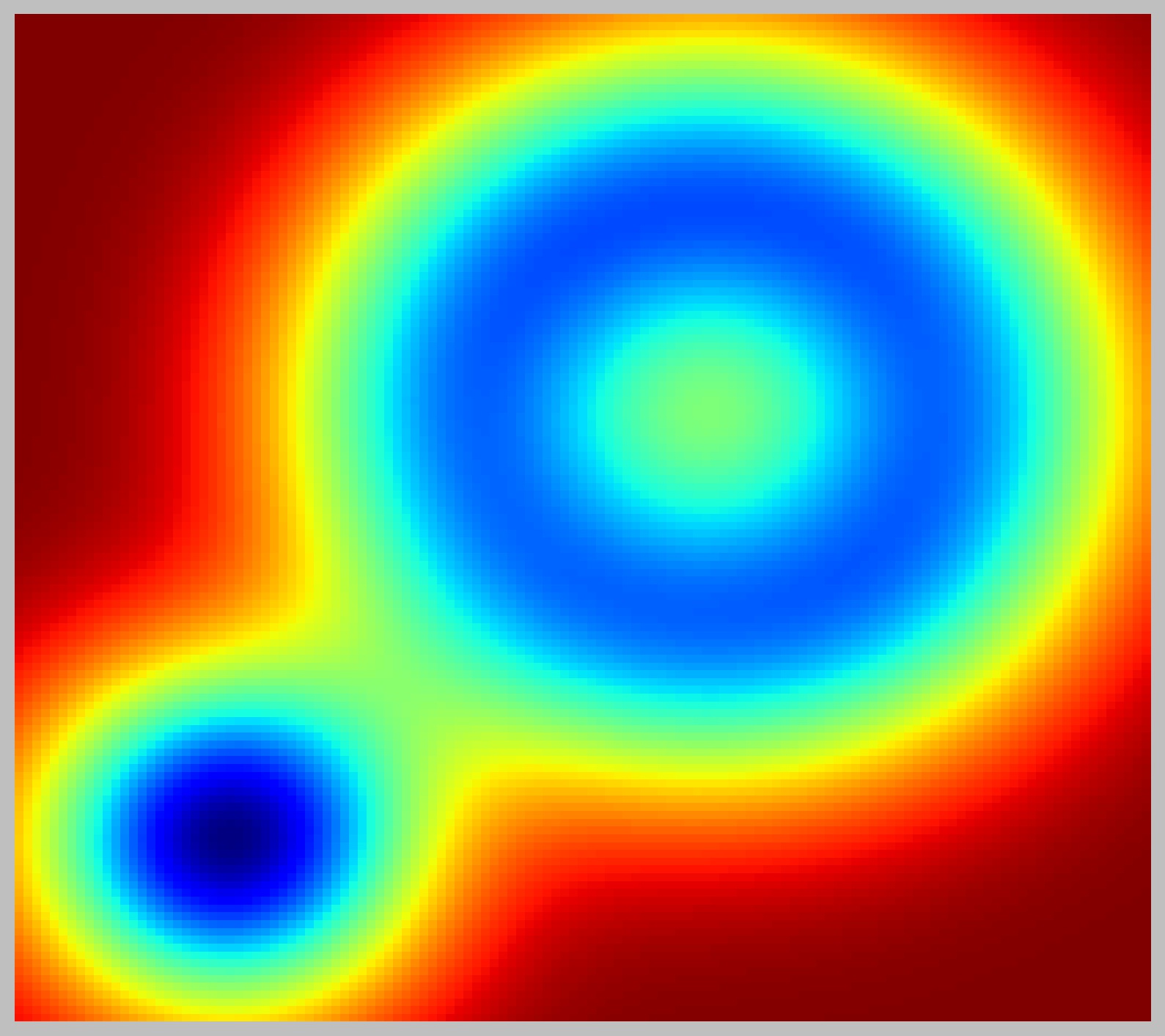} 
\end{tabular} & 
\begin{tabular}{c}
\includegraphics[width=0.14\linewidth]{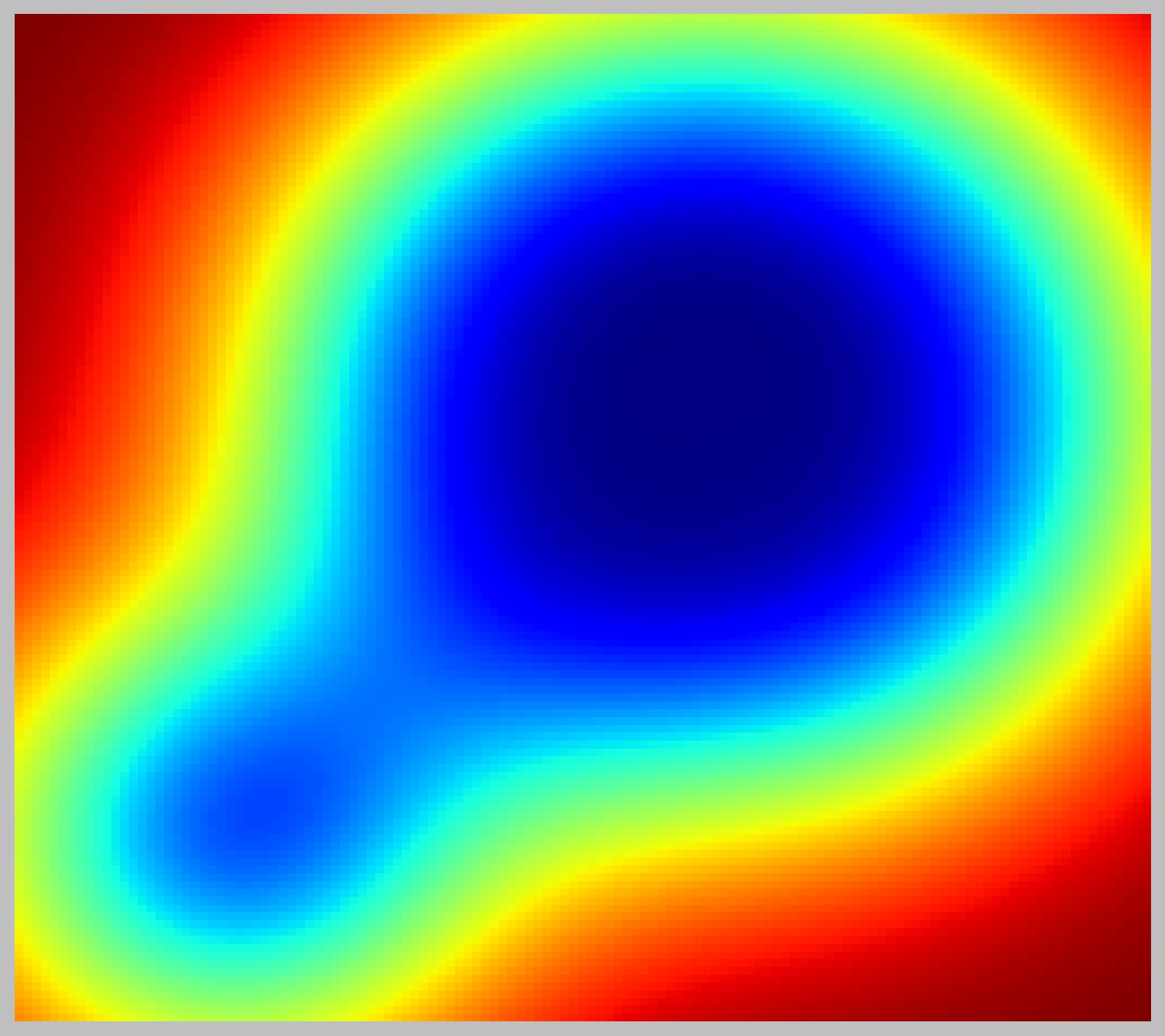} 
\end{tabular} & 
\begin{tabular}{c}
\includegraphics[width=0.14\linewidth]{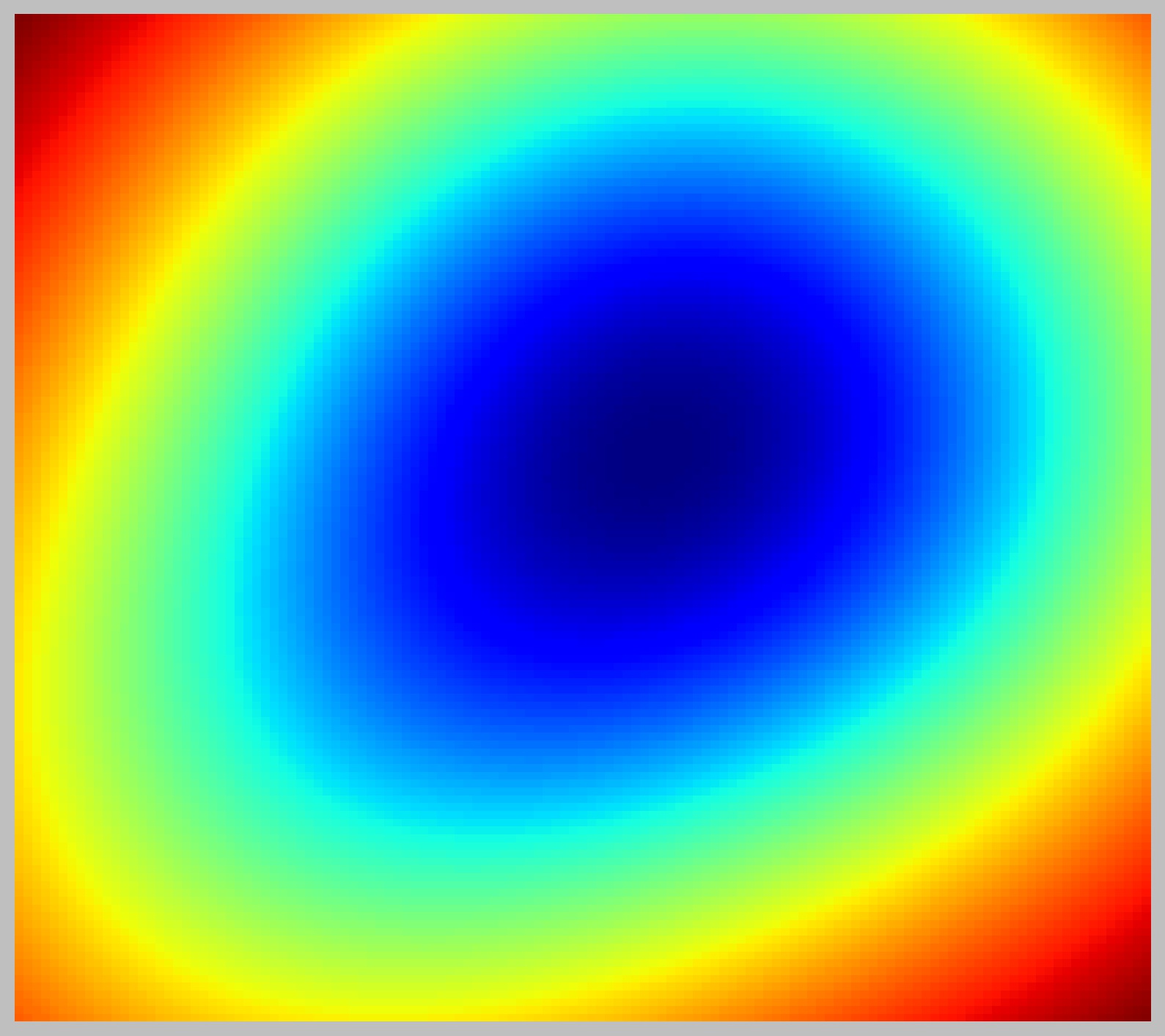} 
\end{tabular} \\
$t = 4.5$ & $t = 9$ &
$t = 30$
\end{tabular}
&
\begin{tabular}{c}
\includegraphics[width=0.33\linewidth]{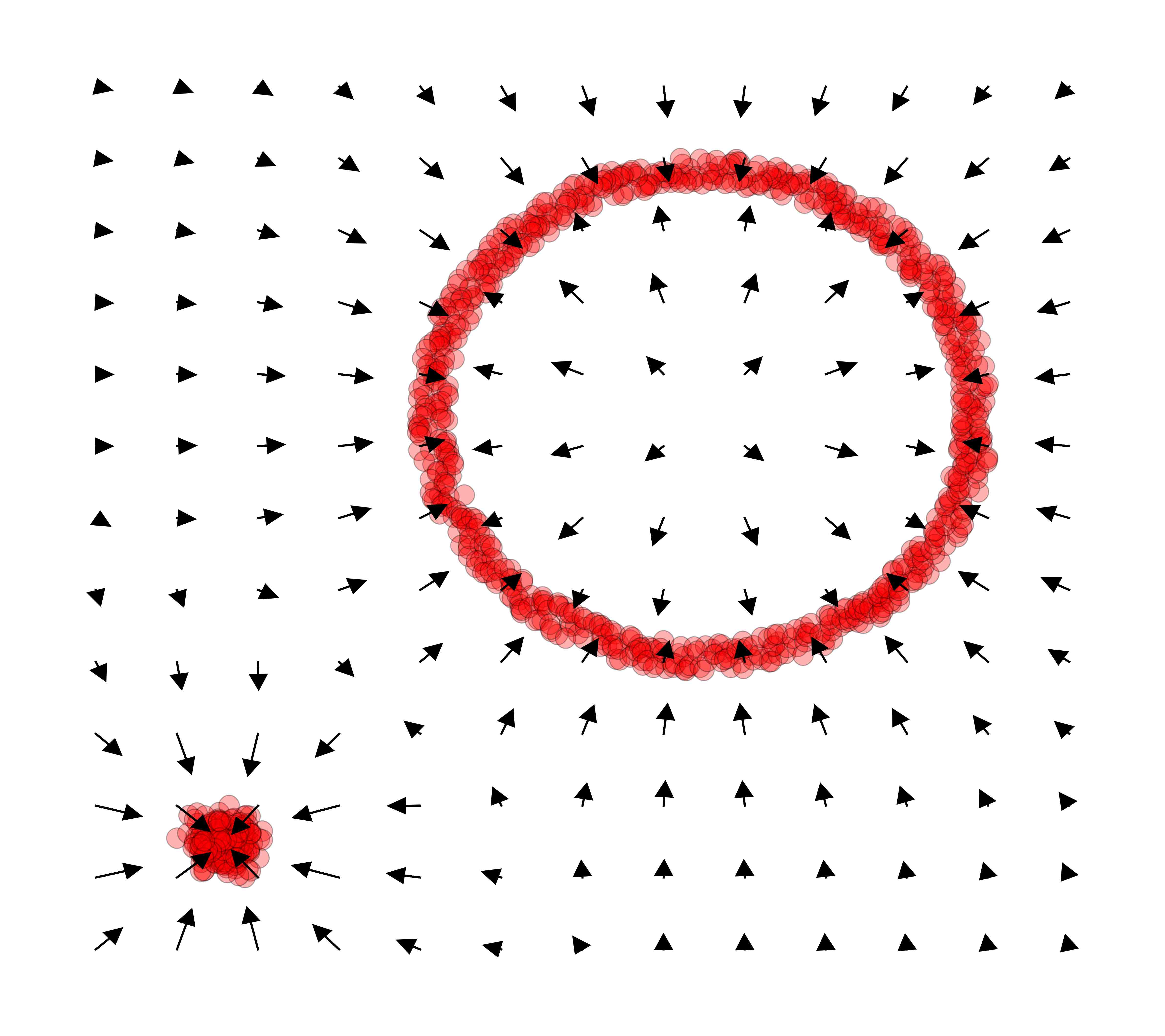}
\\
$t=2$
\end{tabular}
\end{tabular}
\end{center}
\caption{Data points in $\real^2$, heat maps of their diffusion
Fr\'{e}chet functions at increasing scales, and the gradient field
at $t=2$.}
\label{F:frechet2d}
\end{figure}
\end{example}

\begin{example}
This example illustrates the evolution of a dataset consisting of
$n=3158$ points under the  (negative) gradient flow of the Fr\'{e}chet
function of the associated empirical measure at scale $t=0.2$.
Panel (a) in Figure \ref{F:flow} shows the original data and the
other panels show various stages of the evolution towards the
attractors of the system.
\begin{figure}[h!]
\begin{center}
\begin{tabular}{ccc}
   \includegraphics[width=0.25\linewidth]{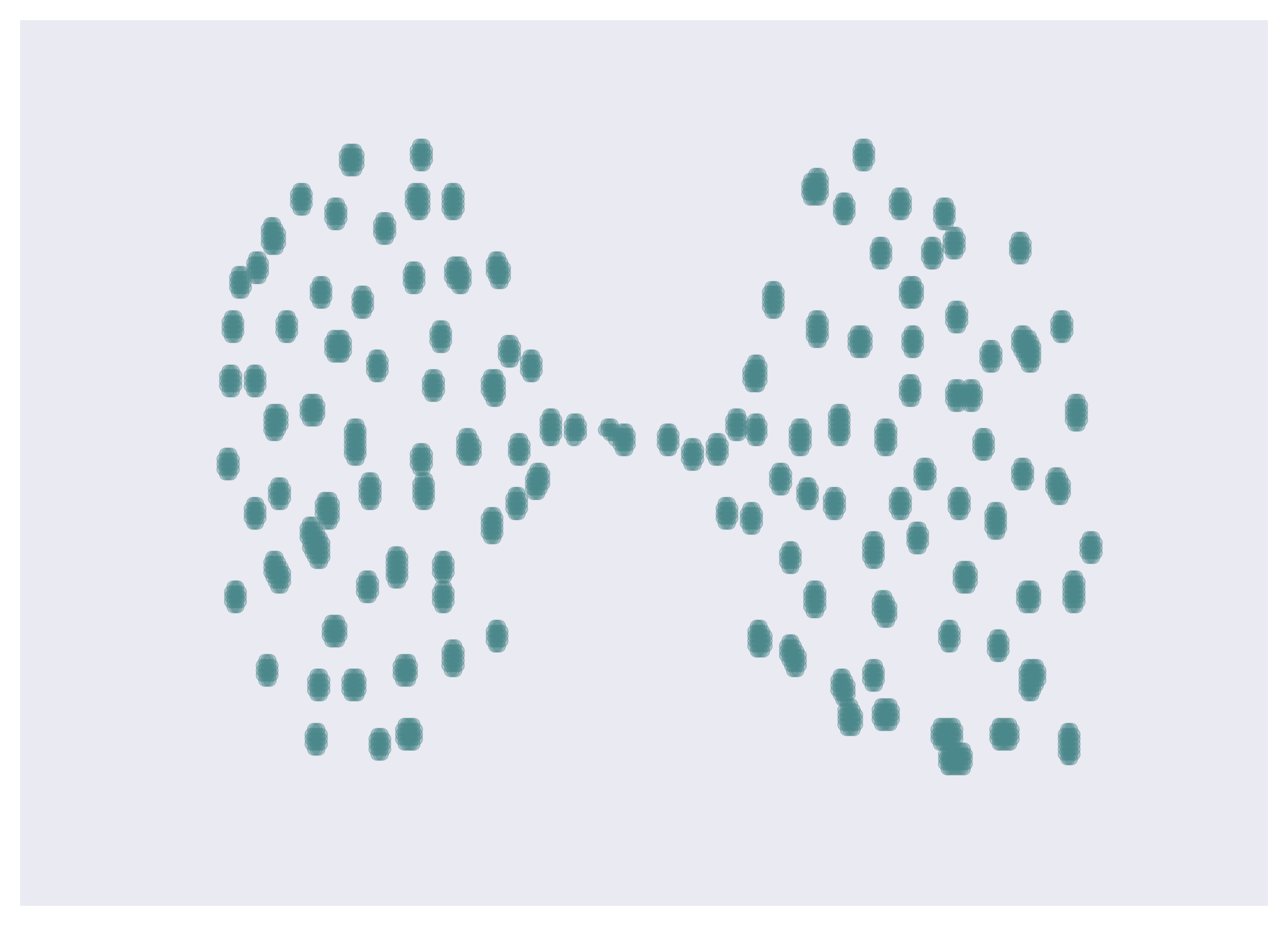} 
   \quad & \quad
   \includegraphics[width=0.25\linewidth]{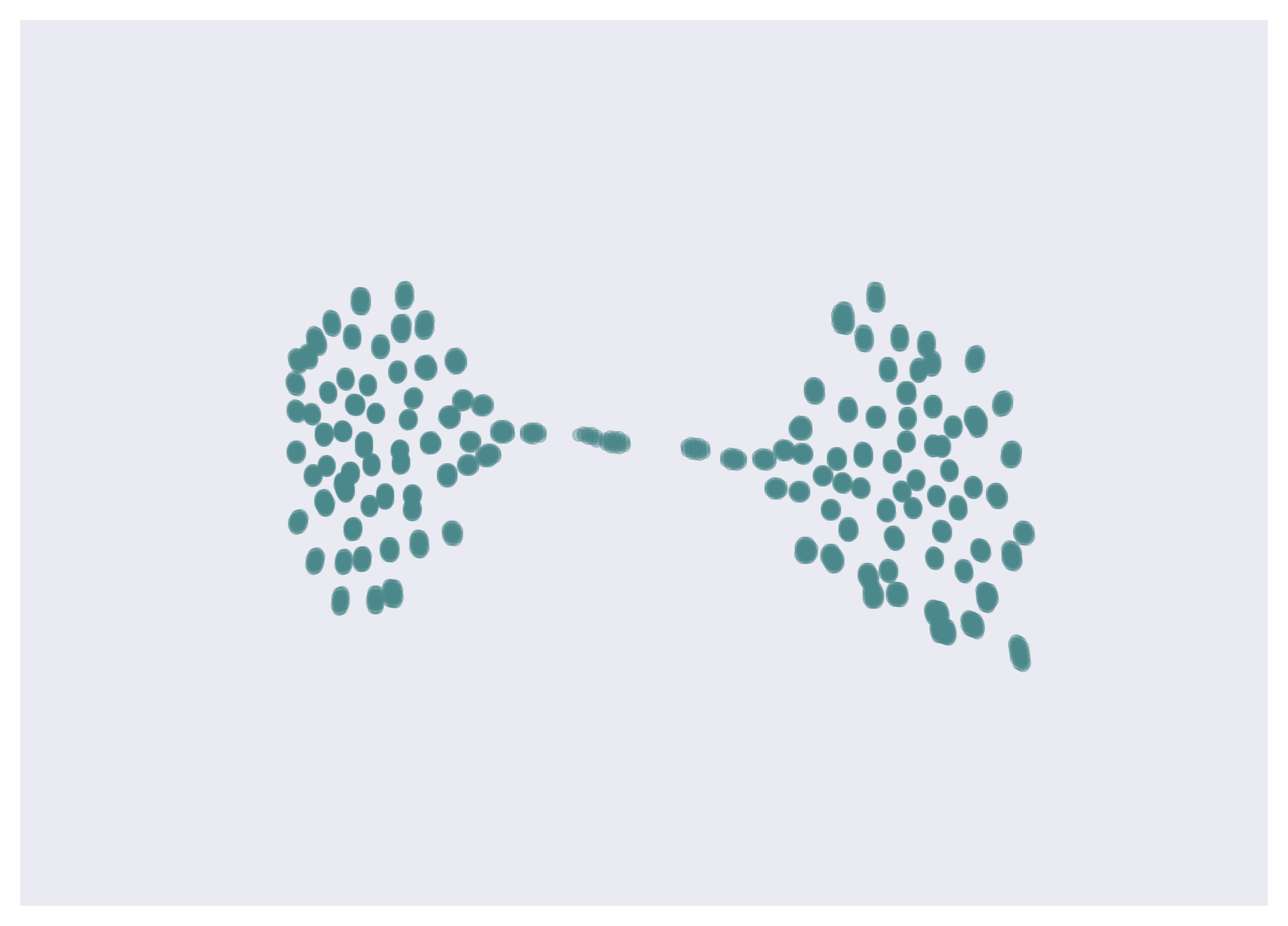} 
   \quad & \quad
   \includegraphics[width=0.25\linewidth]{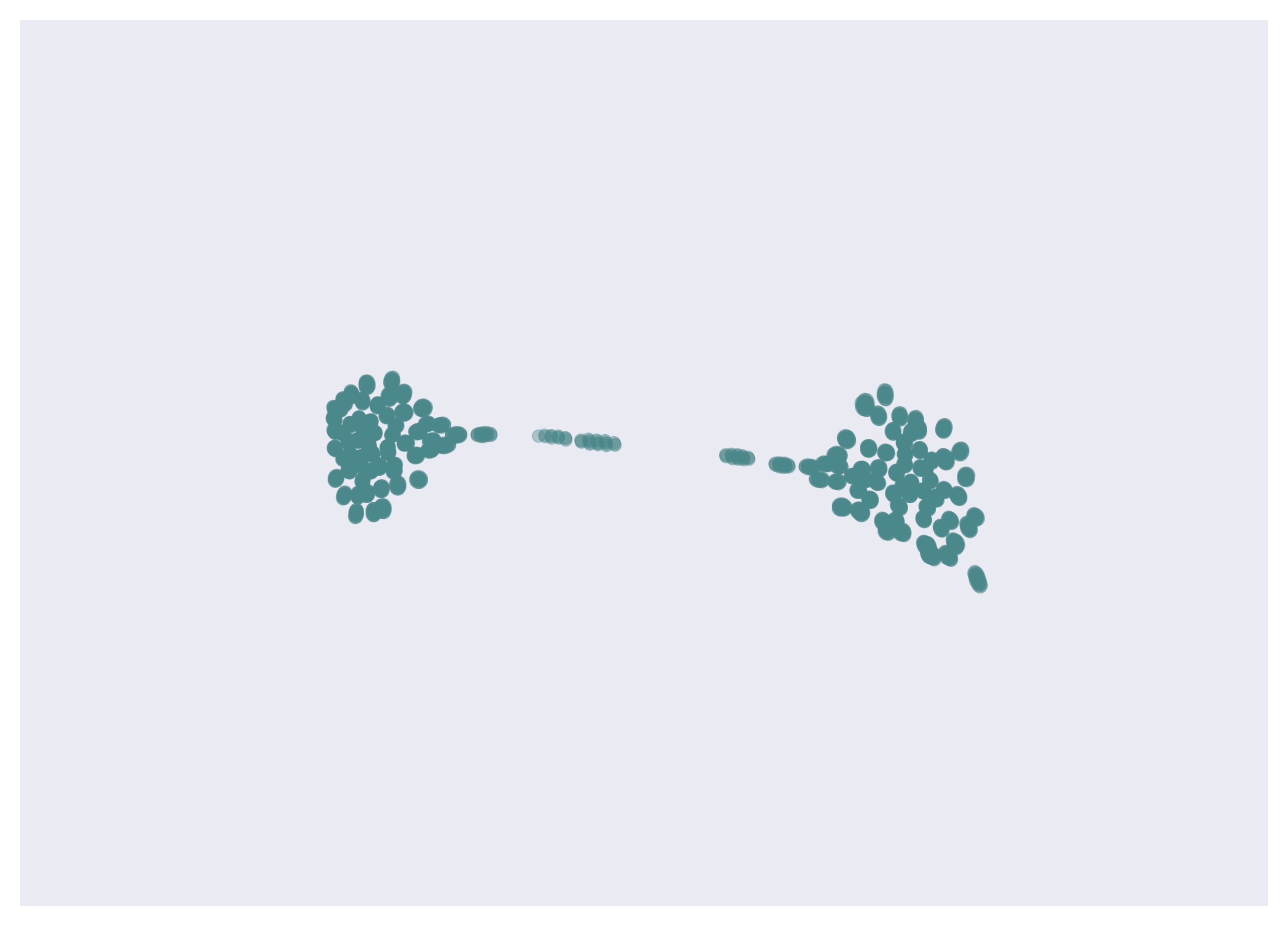} \\
   (a) \quad & \quad (b) \quad & \quad (c) \\
   \includegraphics[width=0.25\linewidth]{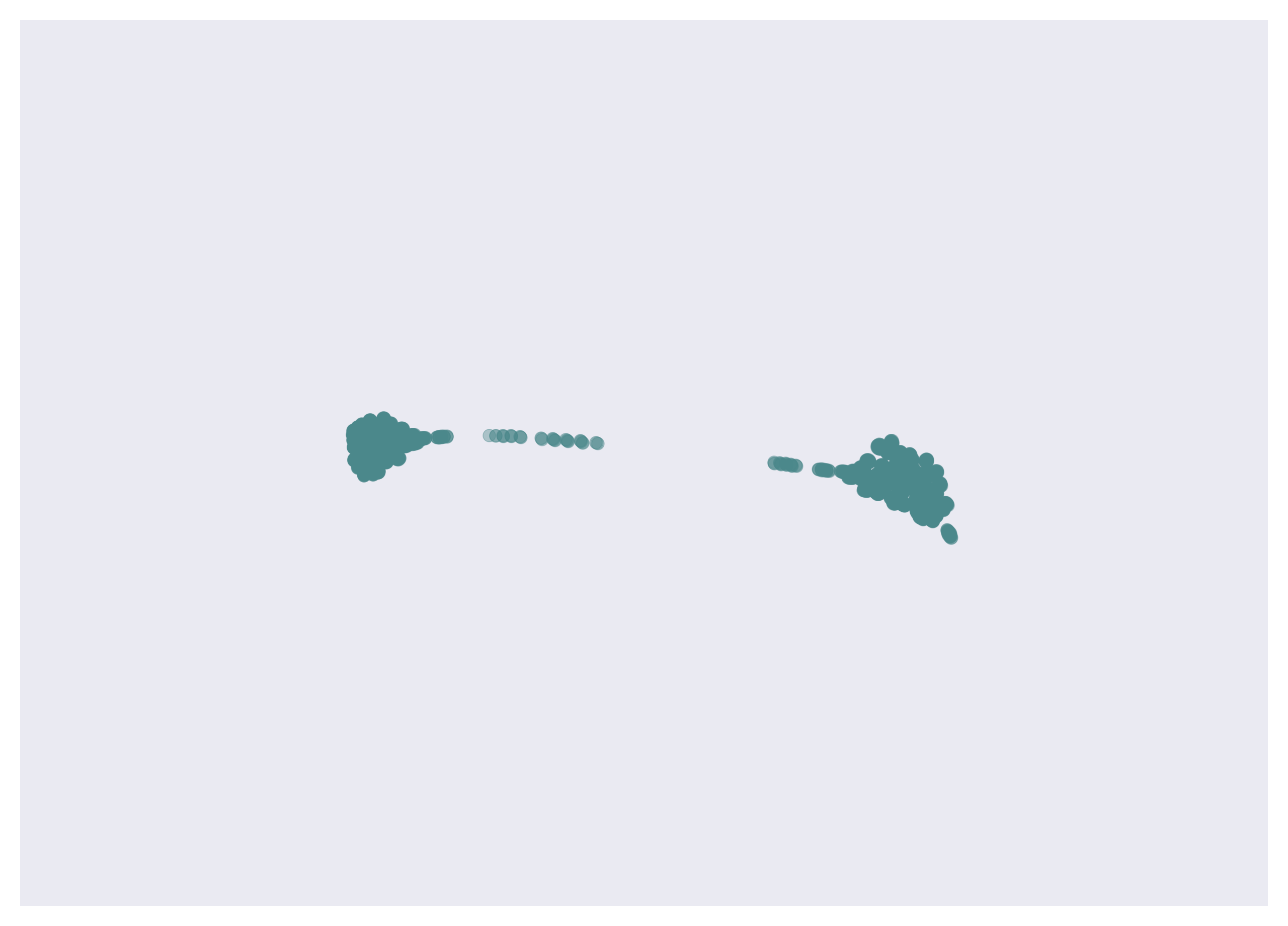} 
   \quad & \quad
   \includegraphics[width=0.25\linewidth]{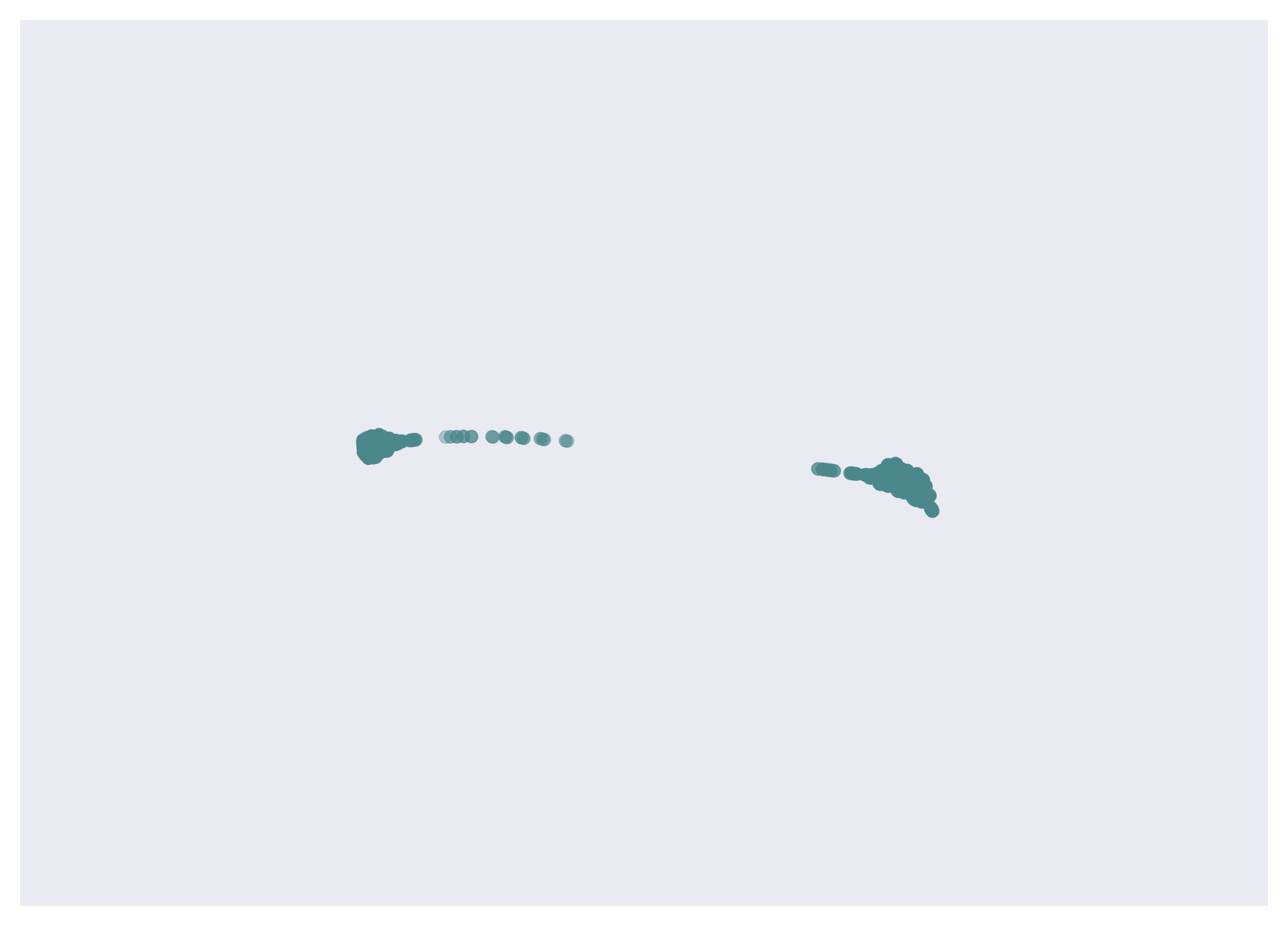} 
   \quad & \quad
   \includegraphics[width=0.25\linewidth]{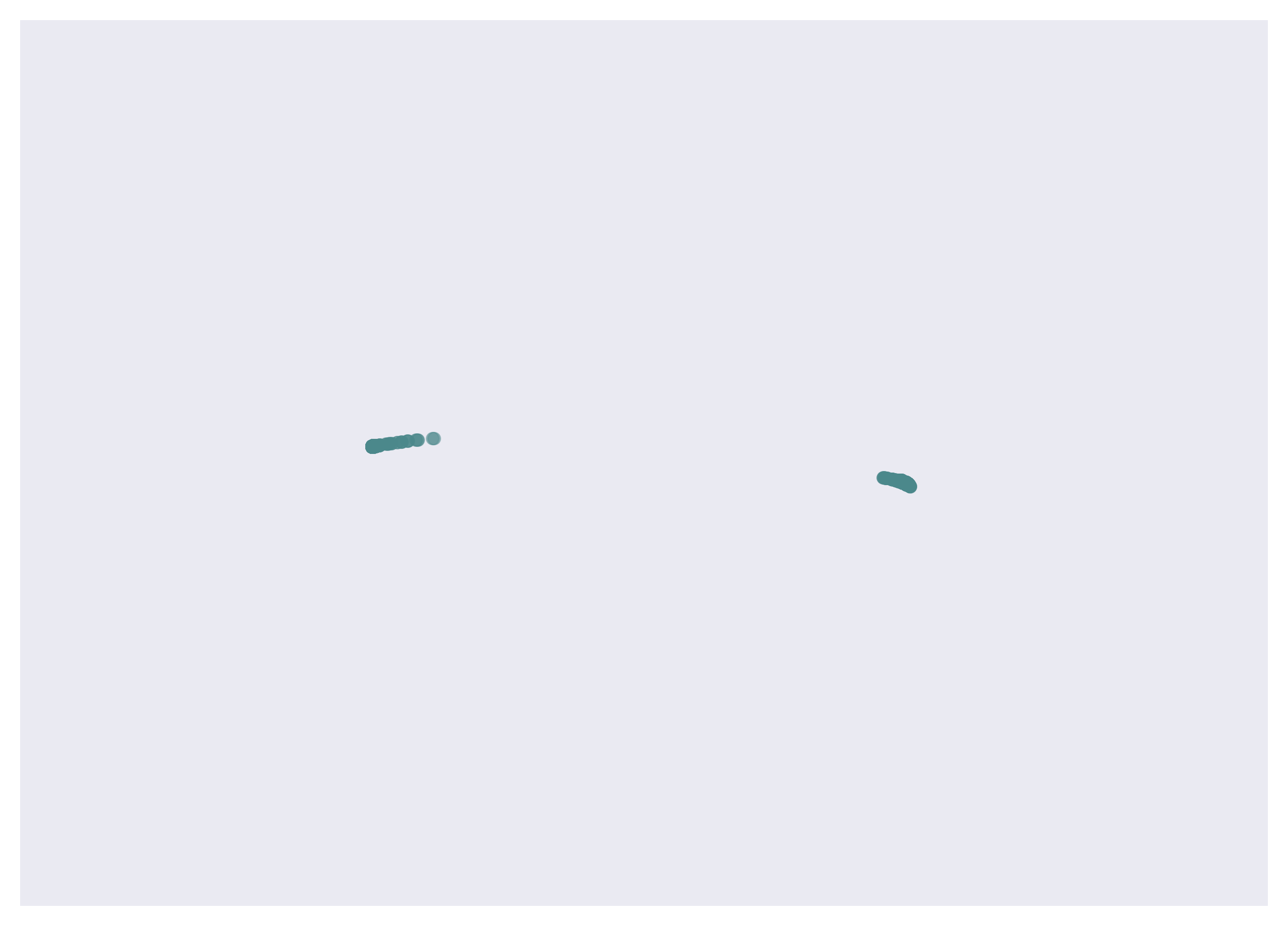} \\
   (d) \quad & \quad (e) \quad & \quad (f) \\
\end{tabular}
\end{center}
\caption{Panel (a) shows the original data and panels (b)--(f) display
various stages of the evolution of the data towards the attractors
of the gradient flow of the Fr\'{e}chet function at scale $t=0.2$.}
\label{F:flow}
\end{figure}
\end{example}

We now prove stability results for $V_{\alpha,t}$ and its gradient
field $\nabla V_{\alpha,t}$, which provide a basis for their use as
robust functional statistics in data analysis. We begin by reviewing
the definition of the Wasserstein distance between two probability
measures.

A {\em coupling\/} between two probability measures $\alpha$
and $\beta$ is a probability measure $\mu$ on $\real^d \times \real^d$
such that $(p_1)_\ast (\mu) = \alpha$ and $(p_2)_\ast (\mu) = \beta$.
Here, $p_1$ and $p_2$ denote the projections onto the first and second
coordinates, respectively. The set of all such
couplings is denoted $\Gamma (\alpha, \beta)$.

For each $p \in [1, \infty)$, let $\borel_p (\real^d)$ denote the collection
of all Borel probability measures $\alpha$ on $\real^d$ whose
$p$th moment $M_p (\alpha) = \int \|y\|^p \, d\alpha (y)$ is finite.

\begin{definition}[cf.\cite{villani09}] \label{D:wass}
Let $\alpha, \beta \in \borel_p (\real^d)$.
The $p$-Wasserstein distance $W_p (\alpha, \beta)$ is defined as
\begin{equation} \label{E:wass}
W_p (\alpha, \beta) = \left(\inf_\mu \iint\limits_{\real^d \times \real^d}
\|x-y\|^p d\mu(x,y)\right)^{1/p} \,,
\end{equation}
where the infimum is taken over all $\mu \in \Gamma (\alpha, \beta)$.
\end{definition}
It is well-known that the infimum in \eqref{E:wass} is realized
by some $\mu \in \Gamma (\alpha, \beta)$ \cite{villani09}. 
Moreover, if $p > q \geq 1$, then
$\borel_p (\real^d) \subset \borel_q (\real^d)$ and
$W_q (\alpha, \beta) \leq W_p (\alpha, \beta)$.

The following lemma will be useful in the proof of the stability results. 
Part (a) of the lemma is a special case of Lemma 1 of \cite{diaz1}. Let
$K_t \colon \real^d \to \real$ be the heat kernel centered at zero.
In the notation used in \eqref{E:gauss}, $K_t (y) = G_t (0,y)$.

\begin{lemma} \label{L:gauss}
Let  $\displaystyle C_d(t) = (4\pi t)^{d/2}$. For any $y_1, y_2 \in \real^d$
and $t>0$, we have:
\begin{itemize}
\item[(a)]
$\displaystyle \left| K_t (y_1) - K_t (y_2) \right| \leq
\frac{\|y_1- y_2\|}{C_d(t)\sqrt{2t\cdot e}}$ ;
\item[(b)]
$\displaystyle \left \Vert y_1 K_t (y_1) - y_2 K_t(y_2)  \right \Vert \leq
\frac{e+2}{e C_d (t)} \|y_1- y_2\|$.
\end{itemize}
\end{lemma}
\begin{proof}
(a) For each $\xi \in[0,1]$, let $y(\xi) = \xi y_1 +(1-\xi) y_2$. Then,
\begin{equation}
\begin{split}
\| K_t (y_1) - K_t (y_2) \| &=
\left| \int_0^1 \frac{d}{d\xi} K_t(y(\xi))\,d\xi \right|
\leq \int_0^1 \left| \frac{d}{d\xi} K_t(y(\xi)) \right| \, d\xi \\
&= \int_0^1 \left| \nabla K_t(y(\xi))\cdot (y_1- y_2) \right| \, d\xi \\
&\leq \|y_1- y_2\| \int_0^1 \left\| \nabla K_t(y(\xi)) \right\| \, d\xi \,.
\end{split}
\end{equation}
Note that
\begin{equation}
\| \nabla K_t(y) \| = \|\frac{y}{2t} K_t(y) \| \leq
\frac{\|y\|}{2tC_d (t)} \exp\left(-\frac{\Vert y\Vert^2}{4t}\right)
\leq \frac{1}{C_d (t) \sqrt{2t\cdot e}} \,.
\end{equation}
In the last inequality we used the fact that
$\|y\| \exp\left(-\frac{\Vert y\Vert^2}{4t}\right) \leq \sqrt{\frac{2t}{e}}$.
Hence,
\begin{equation*}
\begin{split}
\left| K_t (y_1) - K_t (y_2) \right| 
&\leq  \frac{\|y_1- y_2\|}{C_d(t) \sqrt{2t\cdot e}} \,.
\end{split}
\end{equation*}

(b) Using the same notation as in (a), 
\begin{equation}
\frac{d}{d\xi} \left[y K_t (y) \right] = 
K_t (y) (y_1-y_2) -  y \frac{K_t (y)}{2t} \, y \cdot (y_2-y_2) \,.
\end{equation}
Hence, 
\begin{equation} \label{E:diff}
\begin{split}
\left\| \frac{d}{d\xi} \left[y K_t (y) \right] \right\| &\leq
\left(K_t (y) + \frac{\|y\|^2}{2t} K_t (y) \right) \|y_2-y_2\| \\
&\leq \frac{1}{C_d (t)} \left(1 + \frac{2}{e}\right) \|y_1-y_2\| \,.
\end{split}
\end{equation}
In the last inequality we used the facts that $K_t (y) 
\leq 1/ C_d (t)$ and $\|y\|^2 K_t (y) \leq 4t/e C_d (t)$.
Writing
\begin{equation} \label{E:diff1}
y_1 K_t (y_1) - y_2 K_t (y_2) = \int_0^1
\frac{d}{d\xi} \left[y K_t (y) \right] \, d\xi \,,
\end{equation}
it follows from \eqref{E:diff} and \eqref{E:diff1} that
\begin{equation}
\|y_1 K_t (y_1) - y_2 K_t (y_2)\| \leq
\frac{1}{C_d (t)} \left(1 + \frac{2}{e}\right) \|y_1-y_2\| \,,
\end{equation}
as claimed.
\end{proof}

\begin{theorem}[Stability of Fr\'{e}chet functions] \label{T:stability1}
Let $\alpha$ and $\beta$ be Borel probability measures on $\real^d$
with diffusion Fr\'{e}chet functions $V_{\alpha,t}$ and $V_{\beta,t}$,
$t>0$, respectively. If $\alpha, \beta \in  \borel_1 (\real^d)$,
then
\[
\left\| V_{\alpha,t} - V_{\beta,t} \right\|_\infty \leq 
\frac{1}{C_d (2t)\sqrt{t\cdot e}}\, W_1 (\alpha,\beta) \,.
\]
\end{theorem}

\begin{proof}
Fix $t>0$ and $x \in \real^d$. Let $\mu \in \Gamma(\alpha,\beta)$
be a coupling such that
\begin{equation}
\displaystyle\iint\limits_{\real^d \times
\real^d}\|z_1 - z_2\| d \mu(z_1, z_2) = W_1 (\alpha,\beta) \,.
\end{equation} Then,
we may write
\begin{equation} \label{E:simplified1}
V_{\alpha,t} (x)=\iint\limits_{\real^d \times \real^d}
d^2_t(z_1,x)\, d\mu(z_1, z_2)
\end{equation}
and
\begin{equation} \label{E:simplified2}
V_{\beta,t} (x)=\iint\limits_{\real^d \times \real^d}
d^2_t(z_2,x)\, d\mu(z_1, z_2).
\end{equation}
By \eqref{E:difdistance},
$d^2_t(z,x)=2/C_d (2t) - 2 G_{2t}(z,x)$. Hence,
\begin{equation} \label{E:simplified3}
V_{\alpha,t} (x) - V_{\beta,t} (x) = -2 \iint\limits_{\real^d \times \real^d}
\left(G_{2t} (z_1,x) - G_{2t} (z_2,x)\right) d\mu(z_1, z_2) \,,
\end{equation}
which implies that
\begin{equation} \label{E:simplified4}
\left \vert V_{\alpha,t} (x)- V_{\beta,t} (x) \right \vert \leq
2 \iint\limits_{\real^d \times \real^d}
\left\vert  G_{2t}(z_1,x) - G_{2t}(z_2,x) \right\vert d\mu(z_1, z_2) \,.
\end{equation}
After translating $\alpha$ and $\beta$, we may assume that $x=0$.
Thus, by Lemma \ref{L:gauss}(a),
\begin{equation} \label{E:simplified5}
\begin{split}
\left| V_{\alpha,t} (x) - V_{\beta,t} (x) \right|
&\leq  \frac{1}{C_d (2t)\sqrt{t\cdot e}}
\iint\limits_{\real^d \times \real^d} \left \Vert z_1-z_2 \right \Vert
d\mu(z_1, z_2) \\
&= \frac{1}{C_d (2t)\sqrt{t\cdot e}} \, W_1 (\alpha,\beta) \,,
\end{split}
\end{equation}
as claimed.
\end{proof}


\begin{theorem}[Stability of gradient fields] \label{T:stability1a}
Let $\alpha$ and $\beta$ be Borel probability measures on $\real^d$
with diffusion Fr\'{e}chet functions $V_{\alpha,t}$ and $V_{\beta,t}$,
$t>0$, respectively. If $\alpha, \beta \in  \borel_1 (\real^d)$,
then
\[
\sup_{x \in \real^d} \left\| \nabla V_{\alpha,t} (x) -
\nabla V_{\beta,t} (x) \right\| \leq 
\frac{e+2}{e C_d(2t)} \, W_1 (\alpha,\beta) \,.
\]
\end{theorem}

\begin{proof}
Let $\mu \in \Gamma(\alpha, \beta)$ be a coupling that
realizes $W_1 (\alpha, \beta)$.
From \eqref{E:simplified3}, it follows that
\begin{equation}
\begin{split}
\| \nabla V_{\alpha,t} (x) &- \nabla V_{\beta,t} (x) \|
\leq 2 \iint \left\| \nabla_x G_{2t}(z_1,x) - \nabla_x G_{2t}(z_2,x) \right\|
d\mu(z_1, z_2) \\
&\leq \frac{1}{2t} \iint\limits
\left\| (x-z_1) G_{2t}(z_1,x) - (x-z_2) G_{2t}(z_2,x) \right\|
d\mu(z_1, z_2) \,.
\end{split}
\end{equation}
We may assume that $x=0$, so we rewrite the previous
inequality as
\begin{equation}
\| \nabla V_{\alpha,t} (x) - \nabla V_{\beta,t} (x) \|
\leq \frac{1}{2t} \iint
\left\| z_1 K_{2t}(z_1) - z_2 K_{2t}(z_2) \right\|
d\mu(z_1, z_2) \,.
\end{equation}
Therefore, by Lemma \ref{L:gauss}(b),
\begin{equation}
\begin{split}
\| \nabla V_{\alpha,t} (x) - \nabla V_{\beta,t} (x) \|
&\leq \frac{e+2}{e C_d(2t)} \iint \|z_1-z_2\| d\mu(z_1, z_2) \\
&= \frac{e+2}{e C_d(2t)} W_1 (\alpha, \beta) \,,
\end{split}
\end{equation}
which proves the theorem.
\end{proof}

\section{Diffusion Fr\'{e}chet Vectors on Networks}
\label{S:nfrechet}

In this section, we define a network analogue of diffusion Fr\'{e}chet
functions and prove a stability theorem. Let
$\xi = [\xi_1 \, \ldots \, \xi_n]^T \in \real^n$ represent a probability distribution
on the vertex set $V = \{v_1, \ldots, v_n\}$ of a weighted
network $K$. For $t>0$, define the {\em diffusion Fr\'{e}chet vector\/} (DFV)
as the vector $F_{\xi,t} \in \real^n$ whose $i$th component is
\begin{equation}
F_{\xi,t} (i) = \sum_{j=1}^n d_t^2 (i,j) \xi_j \,.
\end{equation}

\begin{example}
We illustrate the behavior of DFVs on a social network
of frequent interactions among 62 dolphins in a community living off
Doubtful Sound, New Zealand \cite{lusseau}. The network data was obtained
from the UC Irvine Network Data Repository. All edges are given the
same weight and we consider the uniform distribution on the vertex set.
Figure \ref{F:dolphins} shows maps of the diffusion Fr\'{e}chet vector
at multiple scales. As in the Euclidean case, the profiles of the DFVs reveal
sub-communities in the network and their interactions.
\begin{figure}[ht]
\begin{center}
\begin{tabular}{cc}
   \includegraphics[width=0.35\linewidth]{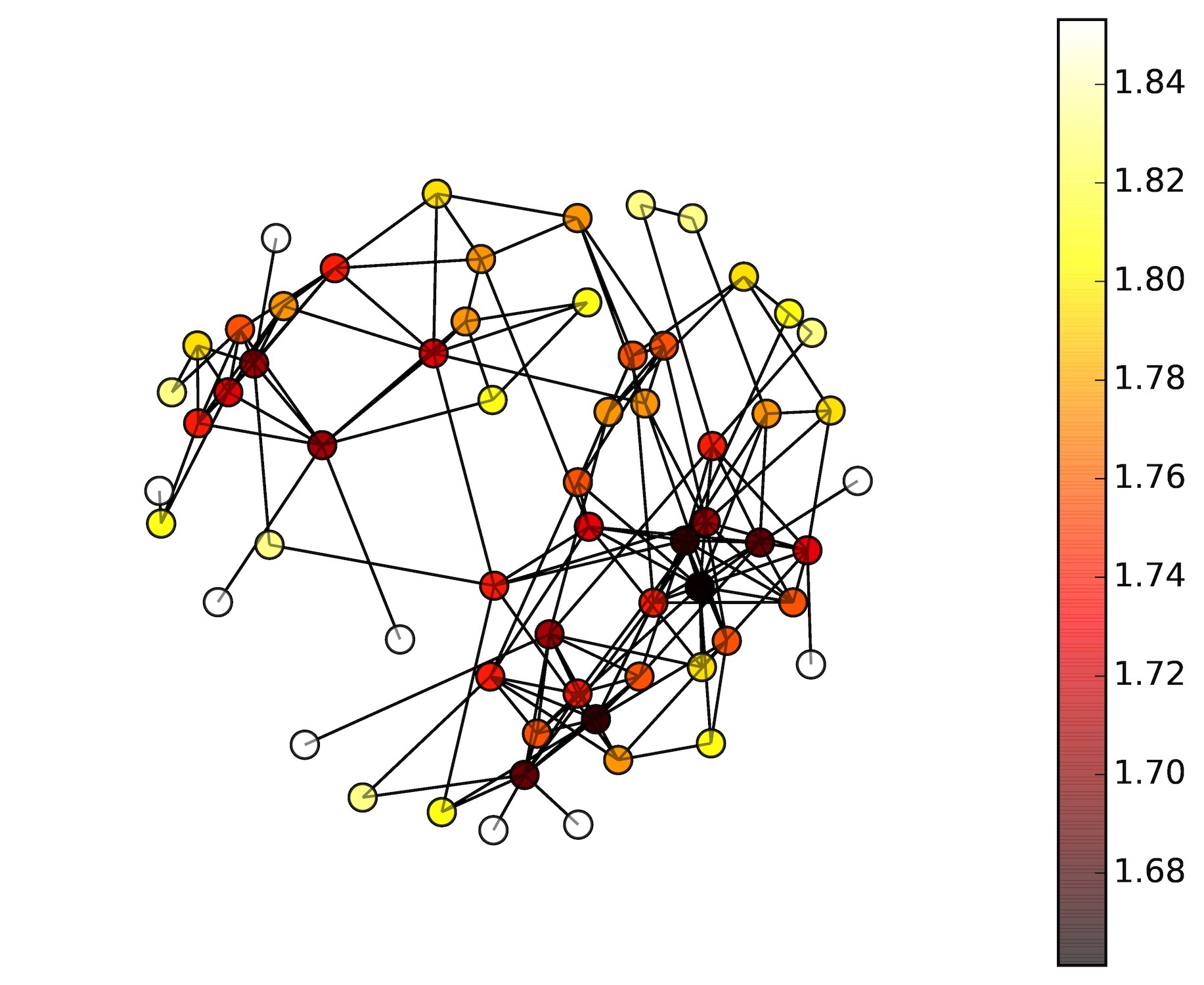} 
   \quad & \quad
   \includegraphics[width=0.35\linewidth]{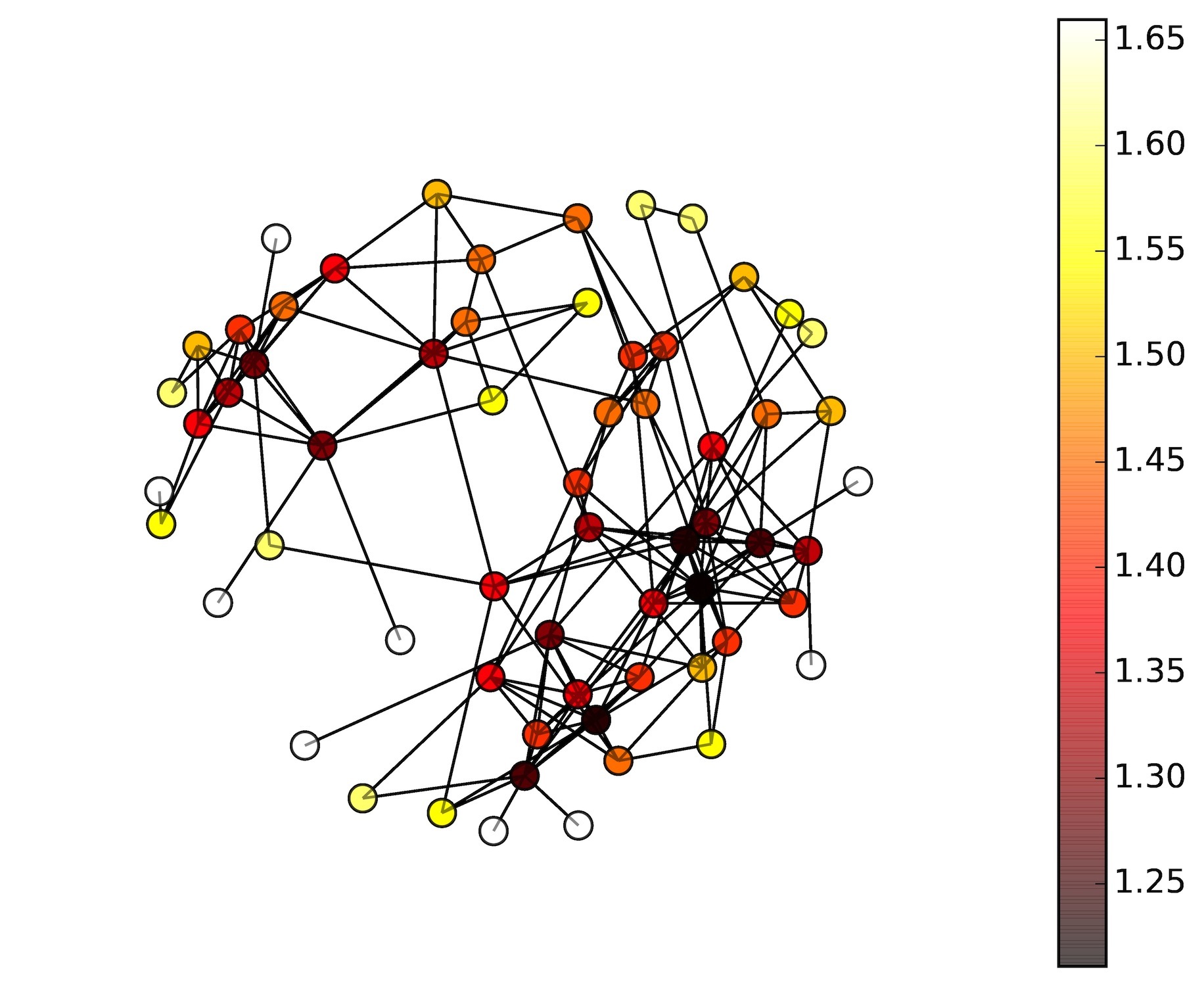} \\
   $t = 0.01$ \quad & \quad $t = 0.03$ \vspace{0.1in} \\
   \includegraphics[width=0.35\linewidth]{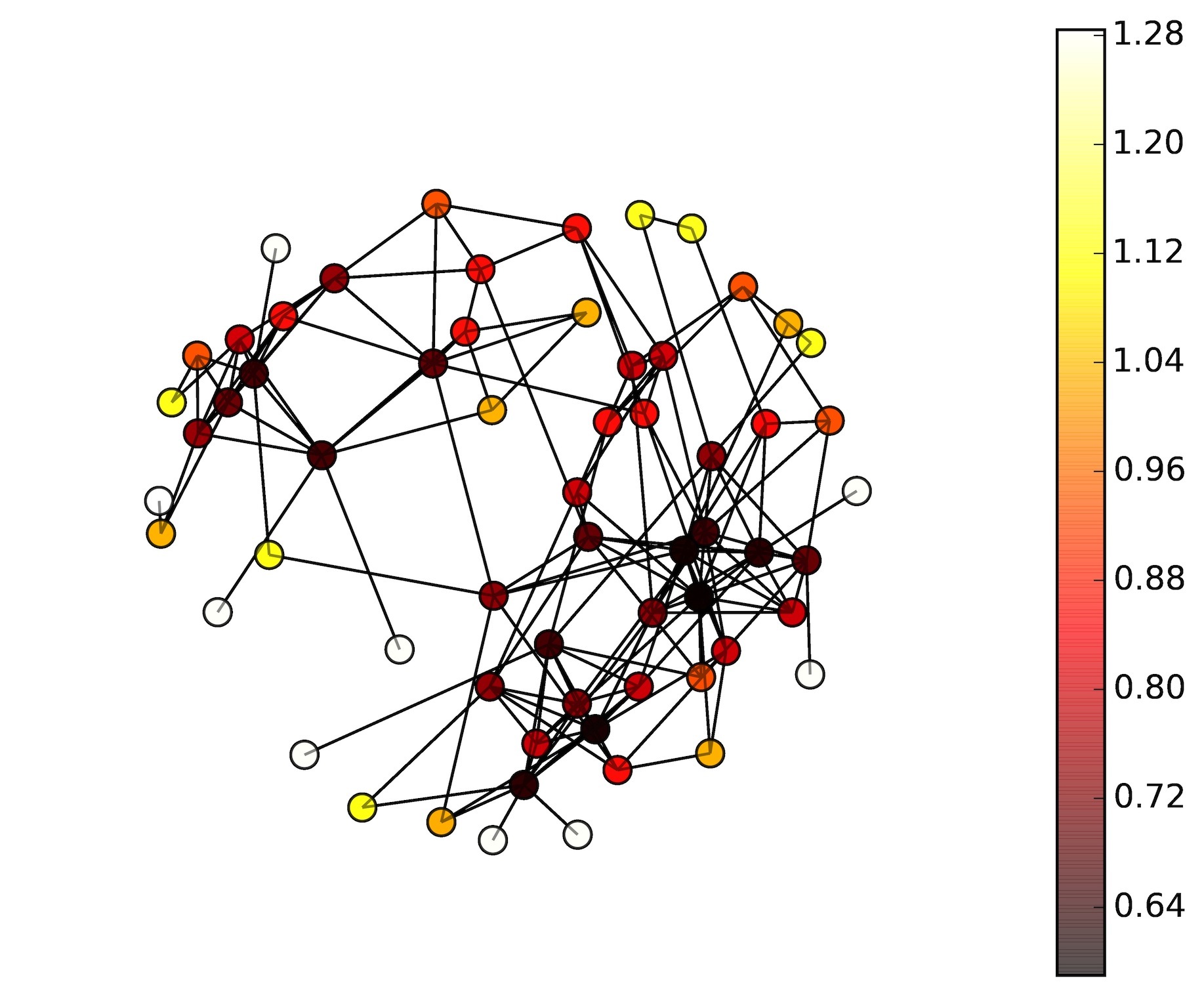} 
   \quad & \quad
   \includegraphics[width=0.35\linewidth]{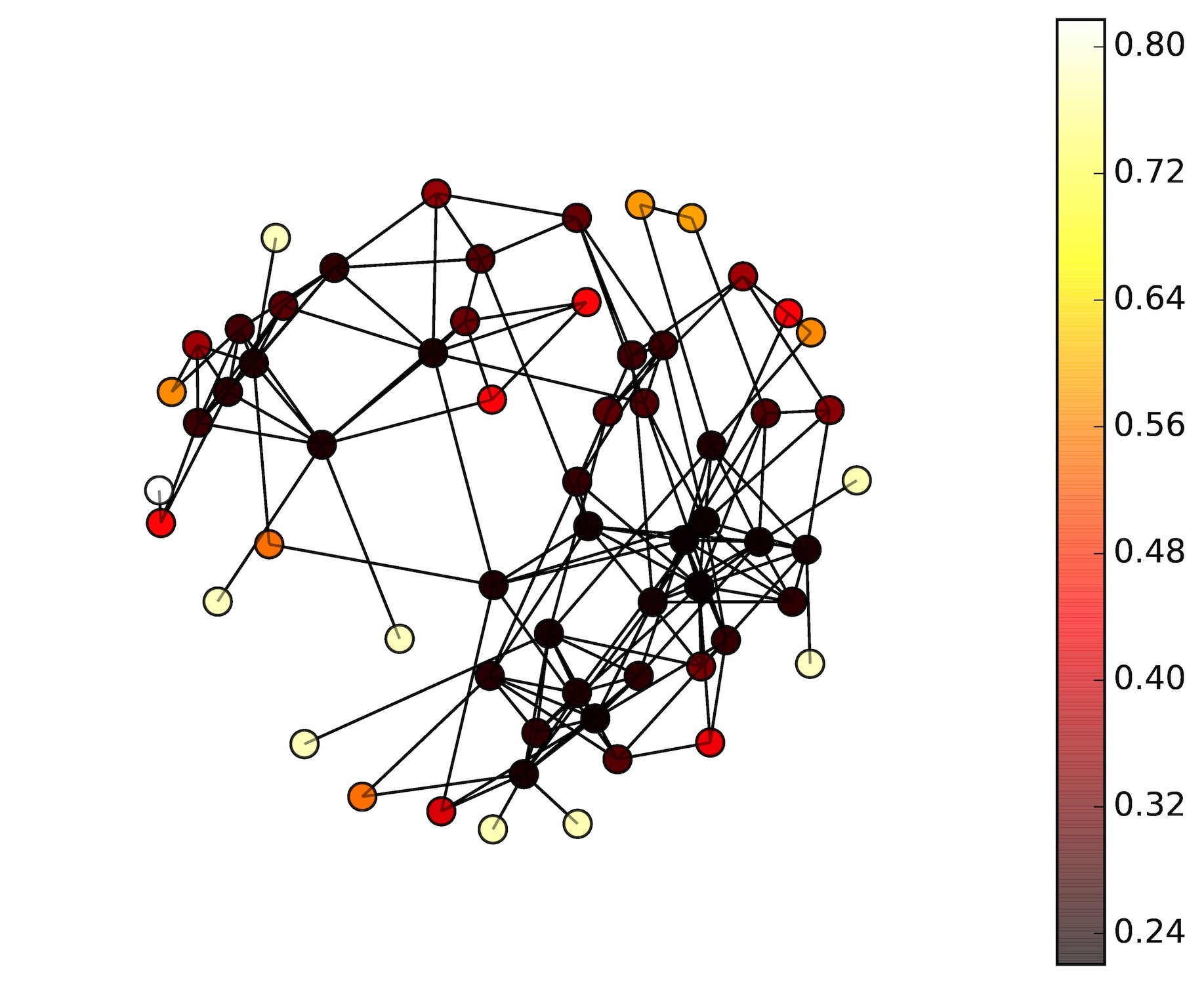} \\
   $t = 0.09$ \quad & \quad $t= 0.3$ \\
\end{tabular}
\end{center}
\caption{Evolution across scales of the diffusion Fr\'{e}chet vector
for the uniform distribution.}
\label{F:dolphins}
\end{figure}
\end{example}

To state a stability theorem for DFVs, we first define the Wasserstein
distance between two probability distributions on the vertex set $V$ of
a weighted network $K$.
This requires a base metric on $V$ to play a role similar to that of the
Euclidean metric in \eqref{E:wass}. We use the {\em commute-time
distance\/} between the nodes of a weighted network (cf. \cite{lovasz96}).
\begin{definition}
Let $K$ be a connected, weighted network with nodes
$v_1, \ldots, v_n$, and let $\phi_1, \phi_2, \ldots, \phi_n$ be an
orthonormal basis of $\real^n$ formed by
eigenvectors of the graph Laplacian with eigenvalues
$0 = \lambda_1 < \lambda_2 \leq \ldots
\leq \lambda_n$. The commute-time distance between $v_i$ and
$v_j$ is defined as
\begin{equation}
d_{CT} (i,j)= \left(
\sum_{k=2}^n \frac{1}{\lambda_k}\left(\phi_k(i)-\phi_k(j)\right)^2
\right)^{1/2}.
\end{equation}
\end{definition}
\noindent
It is simple to verify that $d_{CT}^2 (i,j) = 2 \int_0^\infty d_t^2 (i,j) \, dt$.
\begin{definition}
Let $\xi, \zeta \in \real^n$ represent probability distributions on $V$.
The $p$-Wasserstein distance, $p \geq 1$, between $\xi$ and
$\zeta$ (with respect to the commute-time distance on $V$) is
defined as
\begin{equation}
W_p (\xi, \zeta) = \min_{\mu \in \Gamma (\alpha, \beta)}
\left( \sum_{j=1}^n \sum_{i=1}^n d_{CT}^p(i,j) \mu_{ij} \right)^{1/p} \,,
\label{E:wasserstein-netw}
\end{equation}
where $\Gamma (\xi, \zeta)$ is the set of all probability measures
$\mu$ on $V \times V$ satisfying $(p_1)_\ast (\mu) = \xi$ and
$(p_2)_\ast (\mu) = \zeta$.
\end{definition}

\begin{theorem}
Let $\xi, \zeta \in \real^n$ be probability distributions on the
vertex set of a connected, weighted network. For each $t>0$, their diffusion
Fr\'{e}chet vectors satisfy
\begin{equation}
\|F_{\xi,t} - F_{\zeta,t} \|_\infty \leq 
4 \sqrt{\frac{\Tr e^{-2t \Delta }-1}{2 et}} \, W_1 (\xi, \zeta) \,.
\end{equation}
\end{theorem} 

\begin{proof}
Fix $t>0$ and a node $v_\ell$.  Let $\mu \in \Gamma(\xi,\zeta)$
be such that
\begin{equation}
\sum_{j=1}^n \sum_{i=1}^n d_{CT}(i,j)\mu(i,j) = W_1 (\xi, \zeta)\,.
\end{equation}
Since $\mu$ has marginals $\alpha$ and $\beta$, we may write
\begin{equation} 
F_{\xi,t}(\ell)=\sum_{j=1}^n \sum_{i=1}^n  d_t^2(i,\ell) \mu_{i,j}
\quad \text{and} \quad
F_{\zeta,t}(\ell)=\sum_{j=1}^n \sum_{i=1}^n  d_t^2(j,\ell) \mu_{i,j} \,,
\end{equation}
which implies that
\begin{equation}
\left\vert F_{\xi,t}(\ell) - F_{\zeta,t}(\ell) \right\vert \leq
\sum_{j=1}^n \sum_{i=1}^n  \left| d_t^2(i, \ell) -
d_t^2(j, \ell)\right|\mu_{i,j} \,.
\end{equation}
By Lemma \ref{L:commute} below,
\begin{equation} \label{E:ineq1}
\begin{split}
\left\vert F_{\xi,t}(\ell) - F_{\zeta,t}(\ell) \right\vert \leq &
 4\sqrt{\Tr e^{-2\Delta t}-1} \sum_{j=1}^n \sum_{i=1}^n d_t(i,j)\mu_{i,j} \,.
\end{split}
\end{equation}
Observe that
\begin{equation} \label{E:ineq2}
\begin{split}
d_t^2(i,j) &= \sum_{k=1}^{n}e^{-2\lambda_kt}\left(\phi_k(i)-\phi_k(j)\right)^2 \\
&= \sum_{k=1}^{n}\lambda_k e^{-2\lambda_kt}\frac{1}{\lambda_k}
\left(\phi_k(i)-\phi_k(j)\right)^2 \leq \frac{1}{2e t} d_{CT}^2(i,j)
\end{split}
\end{equation}
since $\lambda_k e^{-2 \lambda_k t} \leq \frac{1}{2et}$.
The theorem follows from \eqref{E:ineq1}
and \eqref{E:ineq2}.
\end{proof}

\begin{lemma} \label{L:commute}
Let $v_i,v_j,v_{\ell}$ be nodes of a connected, weighted network. For
any $t>0$,
\begin{equation*}
\left\vert d_t^2(i,\ell)-d_t^2(j,\ell) \right\vert \leq
4\sqrt{\Tr e^{-2t \Delta} -1}  \, d_t(i,j)\,,
\end{equation*}
where $\Tr$ denotes the trace operator.
\end{lemma}

\begin{proof}
Since the eigenfunction $\phi_1$ is constant, we may write the
diffusion distance as
\begin{equation} \label{E:diffdist}
\begin{split}
d_t^2 (i,\ell) = \sum_{k=2}^n e^{-2\lambda_kt}
\left(\phi_k^2 (i) - 2 \phi_k (i) \phi_k(\ell) + \phi_k^2 (\ell) \right) \,.
\end{split}
\end{equation}
Thus,
\begin{equation}
\begin{split}
d_t^2(i,\ell)-d_t^2(j,\ell) &=
\sum_{k=2}^{n}e^{-2\lambda_kt}\left(\phi_k^2(i)-\phi_k^2(j)\right) \\
&-2 \sum_{k=2}^{n}e^{-2\lambda_kt}\phi_k(\ell)\left(\phi_k(i)-\phi_k(j)\right) \\
&=  \sum_{k=2}^{n}e^{-2\lambda_kt} (\phi_k(i)+\phi_k(j)) (\phi_k(i)-\phi_k(j)) \\
&- 2\sum_{k=2}^{n}e^{-2\lambda_kt}\phi_k(\ell)\left(\phi_k(i)-\phi_k(j)\right) \,.
\end{split}
\end{equation}
Since each $\phi_k$ has unit norm,
$|\phi_k (\ell)| \leq 1$ and  $|\phi_k (i) + \phi_k (j)| \leq 2$ \,.
Therefore,
\begin{equation} \label{E:difflip}
\left\vert d_t^2(i,\ell)-d_t^2(j,\ell) \right\vert  \leq
4 \sum_{k=2}^{n} e^{-2\lambda_kt} \left| \phi_k(i)-\phi_k(j) \right| \,.
\end{equation}
The Cauchy-Schwarz inequality, applied to the vectors
$a = \left( e^{-\lambda_2 t}, \ldots, e^{-\lambda_n t} \right)$ and
$b = \left( e^{-\lambda_2 t} \left| \phi_2(i)-\phi_2 (j) \right|, \ldots,
e^{-\lambda_n t} \left| \phi_n(i)-\phi_n (j) \right| \right)$, yields
\begin{equation} \label{E:cs}
\begin{split}
\sum_{k=2}^{n} e^{-2\lambda_kt} \left| \phi_k(i)-\phi_k (j) \right|
&\leq \sqrt{\sum_{k=2}^{n}e^{-2\lambda_kt}} \,
\sqrt{\sum_{k=2}^{n}e^{-2\lambda_kt}\left(\phi_k(i)-\phi_k(j)\right)^2} \\
&= \sqrt{\Tr e^{-2t \Delta} -1} \, d_t(i,j) \,.
\end{split}
\end{equation}
The lemma follows from \eqref{E:difflip} and \eqref{E:cs}.
\end{proof}

\section{{\em C. difficile\/} Infection and Fecal Microbiota Transplantation}
\label{S:cdi}

This section presents an application to analyses of
microbiome data associated with {\em Clostridium difficile} infection
(CDI). CDI kills thousands of patients every year in healthcare facilities
\cite{kelly}.  Traditionally, CDI is treated with antibiotics, but the drugs
also attack other bacteria in the gut flora and this has been linked to
recurrence of CDI in recovering patients
\cite{vincent}.  Fecal microbiota transplantation (FMT) is a promising
alternative that shows recovery rates close to $90\%$ \cite{bakken, gough}.
Shahinas {\em et al.}\,\cite{shahinas} and Seekatz {\em et al.}\,\cite{seekatz}
have used 16S rRNA sequencing to estimate the abundance of the various
bacterial taxa living in the gut of healthy donors and CDI patients. These
studies have reported a reduced diversity in the bacterial communities
of CDI patients and abundance scores after FMT treatment that are closer
to those for healthy donors. The studies, however, have focused on counts
of bacterial taxa, disregarding interactions in the bacterial communities. To
account for these, we employ diffusion Fr\'{e}chet vectors to examine
differences in pre-treatment and
post-treatment fecal samples and the effect of FMT in the composition
of the gut flora. The analysis is based on metagenomic data for 17 patients
(paired pre-FMT and post-FMT) and 7 donor samples, selected from data
collected by Lee {\em et al.}\,\cite{lee}.

\subsection{Metagenomic Data} \label{S:data}

The data comes from a subset of 94 patients treated with fecal
microbiota transplantation by
one of the authors \cite{lee}, covering the period 2008--2012. From this,
17 patients were selected, not randomly, for sequencing. The protocol
followed for obtaining consent consisted of first sending a written letter
to each patient asking permission to further study their already collected
stool samples. In the letter it was stated that there will be a followup
telephone call to the patient verifying that they received the letter and
whether they will provide consent 5-10 business days following the
date of the mailing. All patients in this study were contacted, and all
patients provided written informed consent. Furthermore, this study
and permission protocol was approved by the Hamilton Integrated
Research Ethics Board \#12-3683, the University of Guelph
Research Ethics Board 12AU013 and the Florida State University
Research Ethics IRB00000446.

All {\em C. difficile} infections were confirmed by in-hospital, real-time,
polymerase chain reaction (PCR) testing for the toxin B gene. This study
sequenced the forward V3-V5 region of the 16S rRNA gene from 17 CDI
patients who were treated with FMT(s). A pre-FMT, a corresponding
post-FMT, and 7 samples from four donors, corresponding altogether
to 41 fecal samples were sequenced. 


The bioinformatics software {\tt mothur} was used
as the primary means of processing and quality-filtering reads and
calculating statistical indices of community structure; see \cite{rush}
for a breakdown of the {\tt mothur} processing pipeline.

\subsection{Co-occurrence Networks} \label{S:network}

This study is based on bacterial interactions at the phylum level.
We model the (expected) interactions among the various phyla found
in the healthy human gut by means of a co-occurrence network
\cite{junker} in which each node represents a phylum. An
edge between two phyla is weighted according to the correlation
between their counts, estimated from samples taken from a group
of healthy individuals. More precisely, let $v_1, \ldots, v_n$
be the nodes of the network and $\rho_{ij}$, $i \ne j$, be the correlation
coefficient between the counts for the phyla represented by $v_i$ and $v_j$.
As we are interested in sub-communities
of interactive phyla, the edge between $v_i$ and $v_j$ is weighted by
the absolute correlation $w_{ij} = |\rho_{ij}|$, disregarding whether the
correlation is positive or negative. Since this construction typically yields
a fully connected network, we use the locally adaptive network
sparsification (LANS) technique \cite{foti} to simplify the network,
retaining the most significant interactions (edges) while keeping the
network connected. The following description of LANS is equivalent
to that in \cite{foti}. Define
\begin{equation}
F_{ij}=\frac{1}{n}\sum_{k=1}^{n}\mathbf{1}\{w_{ik}\leqslant w_{ij}\} \,,
\end{equation}
where $\mathbf{1}\{w_{ik}\leqslant w_{ij}\}$ returns $1$ if
$w_{ik}\leqslant w_{ij}$ and $0$ otherwise. For a given pair $(i,j)$, $F_{ij}$
is the probability that the absolute correlation between the counts for a
random phylum and $v_i$ is no larger than $w_{ij}$.  Observe that $F_{ij}$
may differ from $F_{ji}$. For $i \ne j$, the decision as to whether the edge
between $v_i$ and $v_j$ is deleted or preserved is based on a
``significance'' level $0 \leqslant \alpha \leqslant 1$. The edge is
preserved if $1-F_{ij} < \alpha$ or $1-F_{ji} < \alpha$.  Larger values of
$\alpha$ retain more edges of the network. 

The model we develop is based on seven bacterial phyla:
\emph{Actinobacteria}, \emph{Bacteroidetes}, \emph{Firmicutes}, \emph{Fusobacteria}, \emph{Proteobacteria}, \emph{Verrucomicrobia}, and a group of unidentified bacteria treated as a
single phylum labeled Unclassified. Figure \ref{F:graph} shows the
co-ocurrence network obtained after LANS sparsification with $\alpha= 0.1$. 
Dark edge colors indicate a higher level of interaction; that is, larger
edge weights. Note that there is a highly interactive sub-community
comprising \emph{Actinobacteria}, \emph{Verrucomicrobia} and unclassified bacteria.
This network is used in our analyses of variation in the structure of bacterial communities in samples from CDI patients, recovering patients and healthy
individuals, as well as the effects of FMT. (The package \emph{NetworkX} for
the Python programming language was used to depict the
network \cite{netw}.)
\begin{figure}[h!]
\begin{center}
\includegraphics[width=0.4\textwidth]{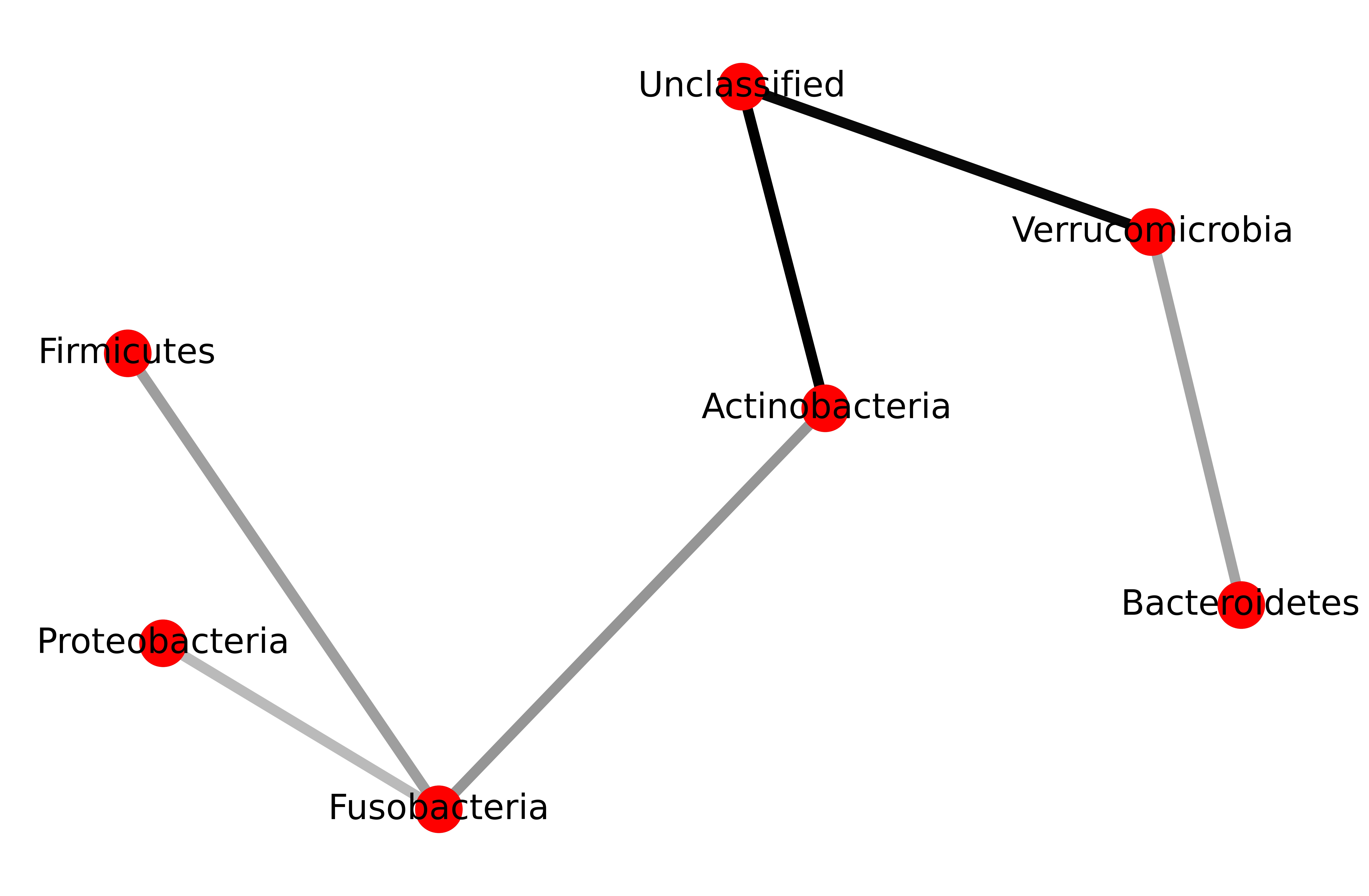}
\end{center}       
\caption{Bacterial phyla co-occurrence network for the human gut
sparsified to significance level $\alpha = 0.1$ with the LANS method.}
\label{F:graph}
\end{figure}

\subsection{Microbiota Analysis} \label{S:analysis}

The co-occurrence network constructed in Section \ref{S:network}
(Figure \ref{F:graph}) from culture data for healthy donors
provides a model for the expected interactions among the seven
bacterial phyla considered in this study. We use the network and
bacterial count data to produce a biomarker $\gamma_t$ that is
effective in characterizing CDI and potentially in monitoring
the effects of FMT treatment. To establish a baseline, we also
construct a biomarker $\beta$ solely based on bacterial counts
and compare it with $\gamma_t$.

Let $v_1, \ldots, v_7$ be the nodes of the co-occurrence network.
For a gut culture sample $S$, let $\xi_i (S)$ be the frequency of
$v_i$ in $S$. Clearly, $\xi_1 (S) + \ldots + \xi_7 (S) = 1$. Our first
method of analysis is based directly on the probability distribution
on the vertex set given by
\begin{equation}
\xi (S) = \left( \xi_1(S) , \ldots , \xi_7 (S) \right) \in \real^7 \,.
\end{equation}
To derive a scalar biomarker $\beta (S)$, we
use culture samples from healthy individuals and CDI patients. We
calculate their 7-dimensional frequency vectors $\xi (S)$ and use linear
discriminant analysis (LDA) to learn an axis in $\real^7$ along
which their scores $\beta (S)$ optimally discriminate healthy samples
from those of CDI patients. 
For a sample $S$, $\beta (S)$ is the score of $\xi (S)$ along the
learned axis. Next, we describe the biomarker $\gamma_t$.
For a sample $S$, let 
\begin{equation}
F^S_t (i) = \sum_{j=1}^7 d^2_t (i,j) \xi_j (S)
\end{equation}
be the diffusion Fr\'{e}chet vector for the distribution $\xi (S)$. 
As before, using training data and applying LDA, we
obtain $\gamma_t$.

\subsection{Results} \label{S:results}

Our analyses were based on 41 gut culture samples comprising 7 healthy
donors, 17 pre-treatment CDI patients (pre-FMT), 4 post-treatment patients
not in resolution (post-NR), and 13 post-treatment patients in resolution
(post-R). Figure \ref{F:freq} shows boxplots for each of the four groups
of the distribution of the phylum frequency data obtained from
metagenomic sequencing of the forty-one samples.
\begin{figure}[h!] 
\begin{center}
\includegraphics[width=4.5in]{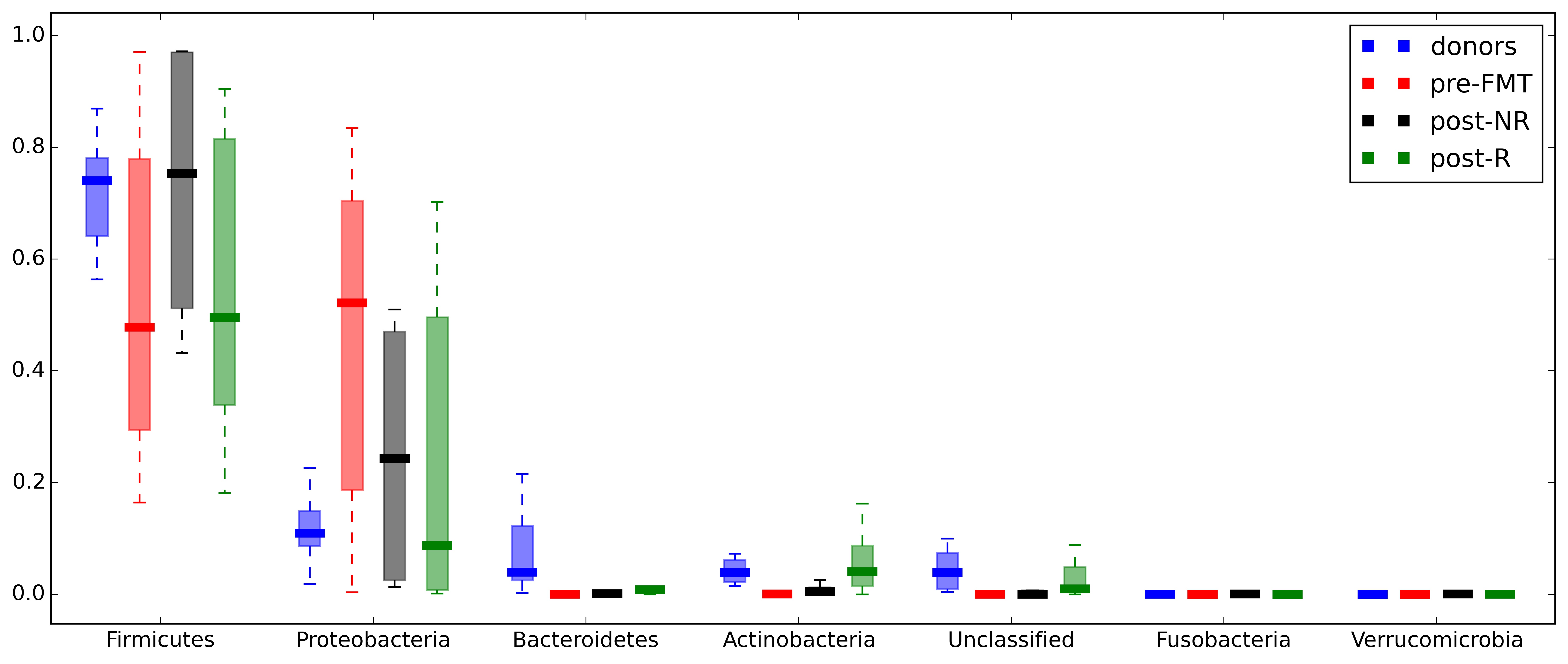}
\end{center}          
\caption{Boxplot of bacterial phylum frequency in the human gut
for healthy donors, pre-FMT patients, post-FMT patients not in
resolution, and post-FMT patients in resolution.}
\label{F:freq}
\end{figure}
Due to the relatively small sample size, we formed a Healthy group
comprising all samples from donors and post-FMT patients in resolution
and a CDI group consisting of all samples from pre-FMT and post-FMT
patients not in resolution. Inclusion of post-R samples in the Healthy
group has the virtue of challenging the biomarkers $\beta$ and
$\gamma_t$ to be sensitive to partial restoration to normal of the
gut flora of recovering patients.

From the frequency data, we constructed $\beta$, as described in
Section \ref{S:analysis}. Linear discriminant analysis yielded an axis
in $\real^7$ along a direction determined by a unit vector whose
loadings are specified in Table \ref{T:loadings}. The loadings revealed that
$\beta$ captures a complex combination of the frequencies of the
various phyla, with only \emph{Fusobacteria} and \emph{Verrucomicrobia} playing
lesser roles due to their low frequencies.
\begin{table}[h!]
\begin{center}
    \begin{tabular}{| c | c | c |} \hline
	Phylum & LDA loadings for $\beta$ & LDA loadings for
	$\gamma_t$ \\ \hline
	\emph{Firmicutes} & -0.412 & -0.079 \\ \hline
	\emph{Proteobacteria} & -0.644 & 0.059 \\ \hline
	\emph{Bacteroidetes} & 0.517 & -0.623 \\ \hline
	\emph{Actinobacteria} & 0.267 & -0.340 \\ \hline
	Unclassified & 0.275 & -0.443 \\ \hline
	\emph{Fusobacteria} & -0.010 & -0.119 \\ \hline
	\emph{Verrucomicrobia} & 0.006 & -0.525 \\ \hline
    \end{tabular}
\end{center}
\caption{Loadings of the directions that determine the axes in
7-D space for $\beta$ and $\gamma_t$.}
\label{T:loadings}
\end{table}
\begin{figure}[ht!]
\begin{center}
\includegraphics[width=4.5in]{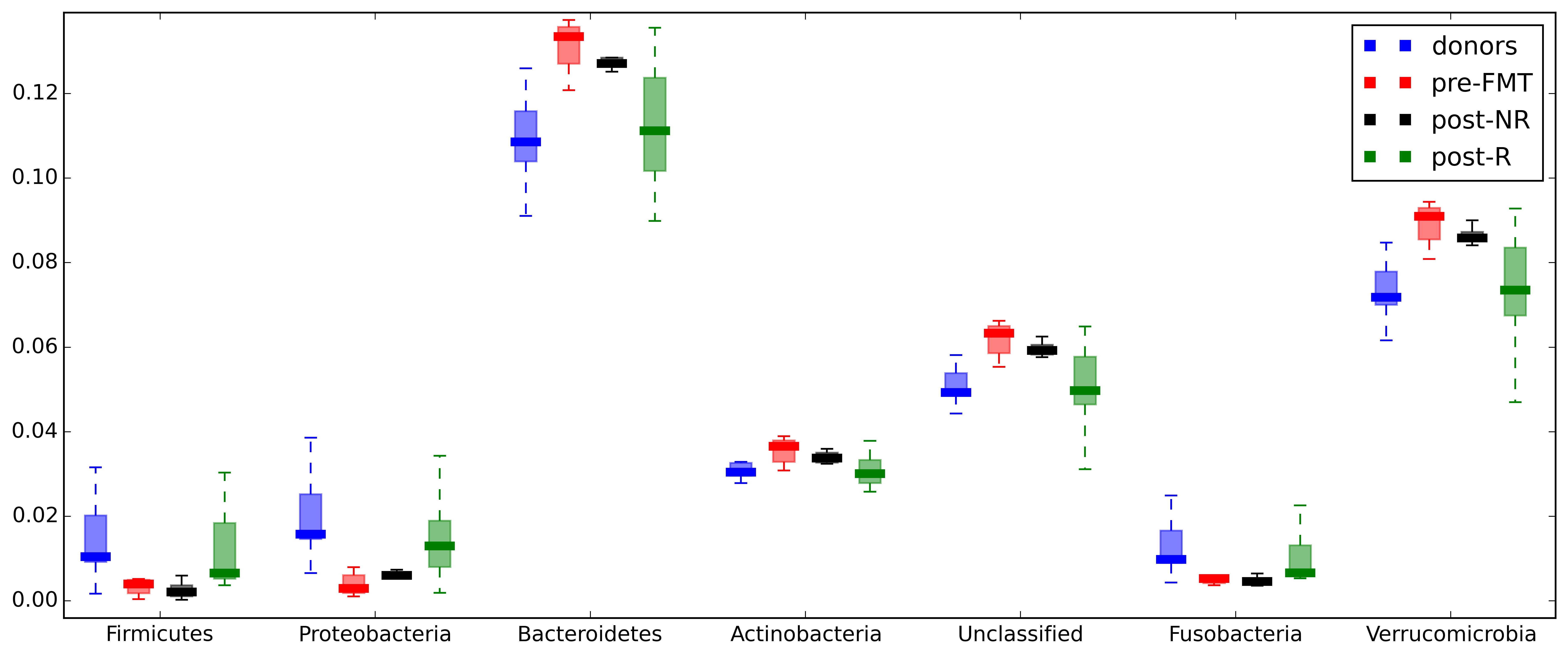}    
\end{center}              
\caption{Boxplot of the values of the diffusion Fr\'{e}chet functions.}
\label{F:dffs}
\end{figure}

A similar analysis was carried out for $\gamma_t$. We tested a range of
values for the significance level $\alpha$ used in network sparsification
and the scale parameter $t$. The values $\alpha = 0.1$ and $t = 7.75$
were selected because they optimized the performance of $\gamma_t$
as measured by the area under its receiver operating characteristic
(ROC) curve \cite{fawcett}, a plot of the true positive rate (sensitivity)
against the false positive rate (1 - specificity) at different threshold
levels. Figure \ref{F:dffs}
shows boxplots of the values of the diffusion Fr\'{e}chet function for
the four groups. Table \ref{T:loadings} shows the loadings for a unit vector
in the direction of the axis in $\real^7$ space associated with $\gamma_t$
that indicate that the composition of the sub-communities
associated with \emph{Bacteroidetes} and \emph{Verrucomicrobia} have a dominant
role in characterizing CDI through $\gamma_t$, followed by Unclassified
and \emph{Actinobacteria}. Note that the model identifies
\emph{Verrucomicrobia} as a key player, whereas its contribution to $\beta$
is minor simply because its count is significantly smaller than the
counts for other phyla. 

Figure \ref{F:roc} provides a comparison of the
ROC curves for $\beta$ and $\gamma_t$, $t = 7.75$, demonstrating
that $\gamma_t$ has a superior performance relative to $\beta$ in
characterizing CDI and recovery from CDI.
\begin{figure}[ht]
\begin{center}
\includegraphics[width=2.5in]{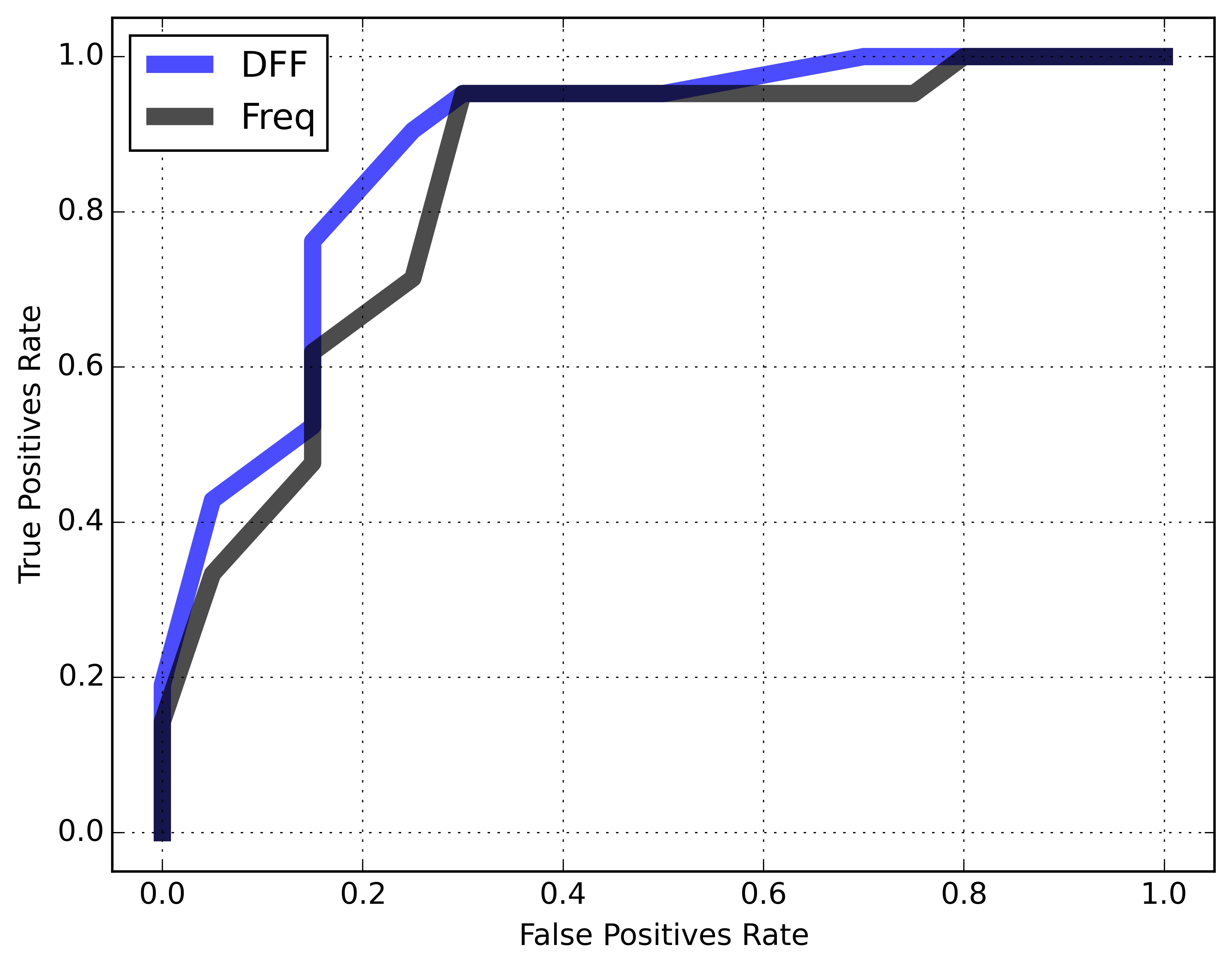}            
\end{center}
\caption{ROC curves for $\beta$ (black) and $\gamma_t$ (blue),
for $t=7.75$.}
\label{F:roc}
\end{figure}
At the threshold value $a = -1.128$,
$\gamma_t$ attained a true positive rate of $83\%$ and a false positive rate
of $20\%$, as estimated from data for 20 healthy and 21 infected samples,
respectively. At $a = -1.108$, the sensitivity increased to $91\%$ with a false
positive rate of $25\%$. As samples from recovering CDI patients whose
gut flora are only partially restored to normal have been included in the healthy
group, we conclude that $\gamma_t$ also shows good sensitivity to the
effects of FMT treatment. In this regard, we note that a significant
portion of the false positives correspond to samples from post-FMT treatment
patients in resolution whose gut flora are only partially restored to normal.

\section{Concluding Remarks} \label{S:remarks}

We introduced diffusion Fr\'{e}chet functions and diffusion Fr\'{e}chet
vectors for probability measures on Euclidean space and weighted
networks that integrate geometric information at multiple spatial
scales, yielding a pathway to the quantification of their shapes.
We proved that these functional statistics are stable with respect
to the Wasserstein distance, providing a theoretical basis for their
use in data analysis. To demonstrate the usefulness of these concepts,
we discussed examples using simulated data and applied DFVs to
the analysis of data associated with fecal microbiota transplantation
that has been used as an alternative to antibiotics in treatment of
recurrent CDI.  Among other things, the method provides a technique
for detecting the presence and studying the organization of
sub-communities at different scales. The approach enables us to
address such problems as quantification of structural variation in
bacterial communities associated with a cohort of
individuals (for example, bacterial communities in the gut of
multiple individuals), as well as associations between structural
changes, health and disease. 

Diffusion Fr\'{e}chet functions may be defined in the broader framework
of metric spaces equipped with a diffusion kernel. However, their
stability in this general setting remains to be investigated. As
remarked in Section \ref{S:efrechet}, the present work suggests
the development of refinements of diffusion Fr\'{e}chet
functions such as diffusion covariance tensor fields for Borel
measures on Riemannian manifolds to make their geometric
properties more readily accessible. 


\section*{Acknowledgements}
This research was supported in part by: NSF grants  DBI-1262351
and DMS-1418007; CIHR-NSERC Collaborative Health Research Projects (413548-2012); and NSF under Grant DMS-1127914 to SAMSI.

\section*{References}

\bibliography{cdi}

\end{document}